\documentclass{article}

\usepackage{arxiv}
\usepackage[authoryear,round]{natbib}

\usepackage{amsmath,amsfonts,bm}

\def\1{\bm{1}}

\DeclareMathAlphabet{\mathsfit}{\encodingdefault}{\sfdefault}{m}{sl}
\SetMathAlphabet{\mathsfit}{bold}{\encodingdefault}{\sfdefault}{bx}{n}

\def\gL{{\mathcal{L}}}

\usepackage{amsmath,amssymb,amsthm,array,booktabs,graphicx,microtype,multirow,url,xcolor}
\usepackage{colortbl}
\usepackage{algorithm,algorithmic}
\usepackage{wrapfig}
\usepackage{tikz}
\usetikzlibrary{arrows.meta,backgrounds,fit,positioning}
\usepackage[section]{placeins}
\usepackage{caption}
\usepackage{hyperref}
\usepackage{enumitem}

\DeclareMathSizes{6.7}{6.7}{5}{5}
\DeclareMathSizes{6.4}{6.4}{5}{5}

\newcolumntype{L}[1]{>{\raggedright\arraybackslash}p{#1}}

\newtheorem{theorem}{Theorem}

\newtheorem{assumption}{Assumption}

\title{Learning to Refer: Inference-Time Collaboration in Decentralized Federated Learning}

\author{
Yuheng Ma$^{1}$\thanks{Authors are listed in alphabetical order by surname.} \qquad Chendi Wang$^{2}$ \qquad Qiong Zhang$^{3}$\thanks{Corresponding author: Qiong Zhang (\texttt{qiong.zhang@mbzuai.ac.ae}).}\\
$^{1}$KLATASDS-MOE, School of Statistics, East China Normal University\\
$^{2}$Department of Statistics and Data Science, Xiamen University\\
$^{3}$Division of Computing and Mathematical Sciences, MBZUAI
}

\date{}
\renewcommand{\headeright}{Preprint}
\renewcommand{\undertitle}{Preprint}
\renewcommand{\shorttitle}{Learning to Refer}
\hypersetup{
pdftitle={Learning to Refer: Inference-Time Collaboration in Decentralized Federated Learning},
pdfauthor={Yuheng Ma, Chendi Wang, Qiong Zhang}
}

\begin{document}

\maketitle

\begin{abstract}
In decentralized federated learning, heterogeneous clients may develop complementary expertise, yet each query is typically evaluated only by the client where it arrives.
To exploit this distributed expertise without global coordination, we formulate inference-time collaboration as budgeted local expert referral, a graph-constrained routing process in which a query can move only between neighboring clients and only visited clients contribute to its final prediction.
\texttt{EdgeDRM} implements this framework by learning a density-based routing potential for each client: differences between neighboring potentials provide the comparisons used to select successive referrals, while local task models generate predictions that are aggregated along the resulting route.
Both training and inference are decentralized, using only neighbor-local information without a central router or graph-wide client scoring.
We establish estimation guarantees for the learned comparisons and characterize how local improvement and graph structure govern prior-weighted density gains from referral.
Experiments on benchmark and real-world medical data show that \texttt{EdgeDRM} outperforms graph-matched decentralized baselines in most settings, with the clearest benefits under sparse connectivity.
\end{abstract}

\section{Introduction}

Federated learning (FL) enables multiple clients to train a shared model without centralizing their raw data.
Conventionally, a server distributes the model and aggregates locally computed updates~\citep{mcmahan2017communication}.
In decentralized FL~\citep{lalitha2019peer}, clients form a communication graph whose nodes represent clients and whose edges specify which pairs can exchange model information directly, with no central server coordinating the network.
The graph determines which clients can collaborate during training, but at inference a query is typically evaluated only by the client where it arrives.
This inference limitation becomes consequential when client data distributions differ, a setting known as statistical heterogeneity.
Consequently, a model may perform well on queries from its local distribution but poorly on those from other clients.

Such variation in local performance creates an opportunity for inference-time collaboration: a query that is poorly handled by its entry client may be better handled by another client.
For instance, a patient may first visit a local hospital that lacks the expertise needed for the case, even though a better-suited specialist is available elsewhere in the hospital network.
Reaching that specialist, however, may require a sequence of local referrals because each physician can refer the patient only to nearby or affiliated hospitals, which may in turn refer the patient onward.
Each referral incurs delay and cost, while the referral path determines which specialists can ultimately be reached.
We call the analogous inference process in decentralized FL \emph{local expert referral}: a query enters at an arbitrary client, may traverse only a limited number of communication edges, and receives predictions only from the clients it visits.
Here, a \emph{referral} is one neighbor-to-neighbor transfer, \emph{routing} is the process that selects successive referrals, and the \emph{route} is the resulting sequence of visited clients.
Unlike global expert selection, local expert referral is sequential and path-dependent: both the referral budget and the graph topology restrict which experts a query can reach.
This raises our central question: \emph{how can client specialization be exploited through graph-constrained local referral to improve on the entry client's prediction?}
Figure~\ref{fig:hospital-referral-overview} compares no routing, centralized routing, and local referral.

\begin{figure*}[t]
  \centering
  \vskip -0.1in
  \begin{minipage}[t]{0.3\textwidth}
    \centering
    \includegraphics[width=\linewidth]{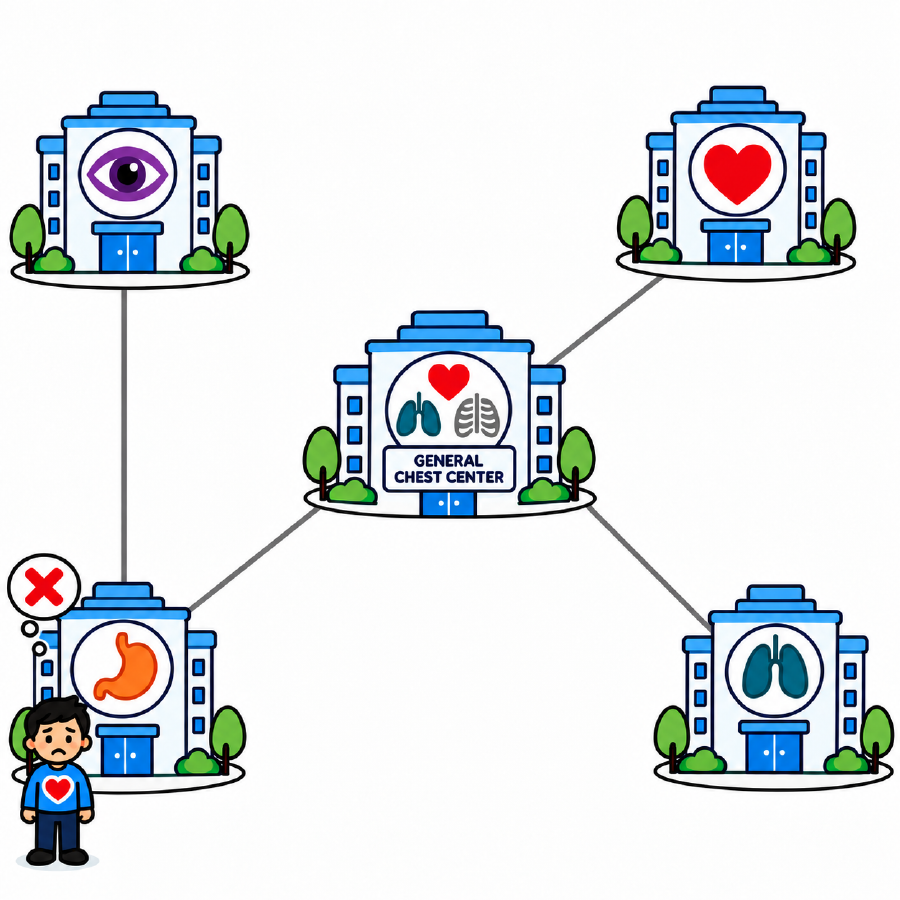}

    {\footnotesize (a) \textbf{No router.}}
  \end{minipage}\hfill
  \begin{minipage}[t]{0.3\textwidth}
    \centering
    \includegraphics[width=\linewidth]{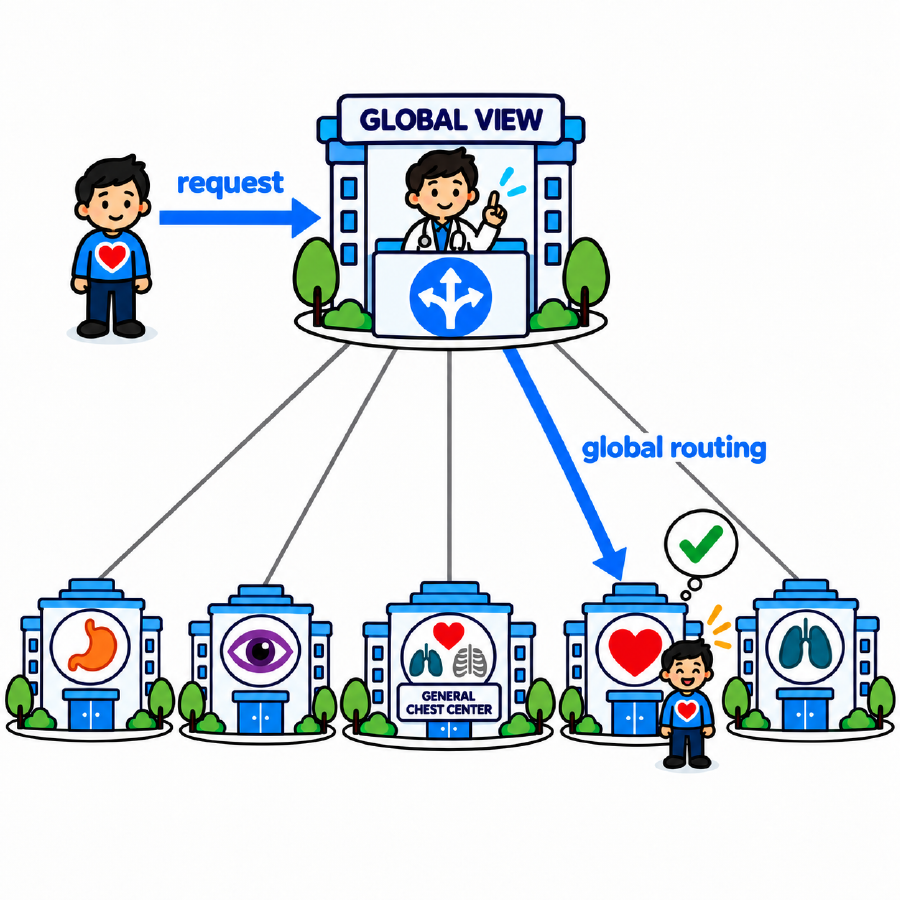}

    {\footnotesize (b) \textbf{Centralized router.}}
  \end{minipage}\hfill
  \begin{minipage}[t]{0.3\textwidth}
    \centering
    \includegraphics[width=\linewidth]{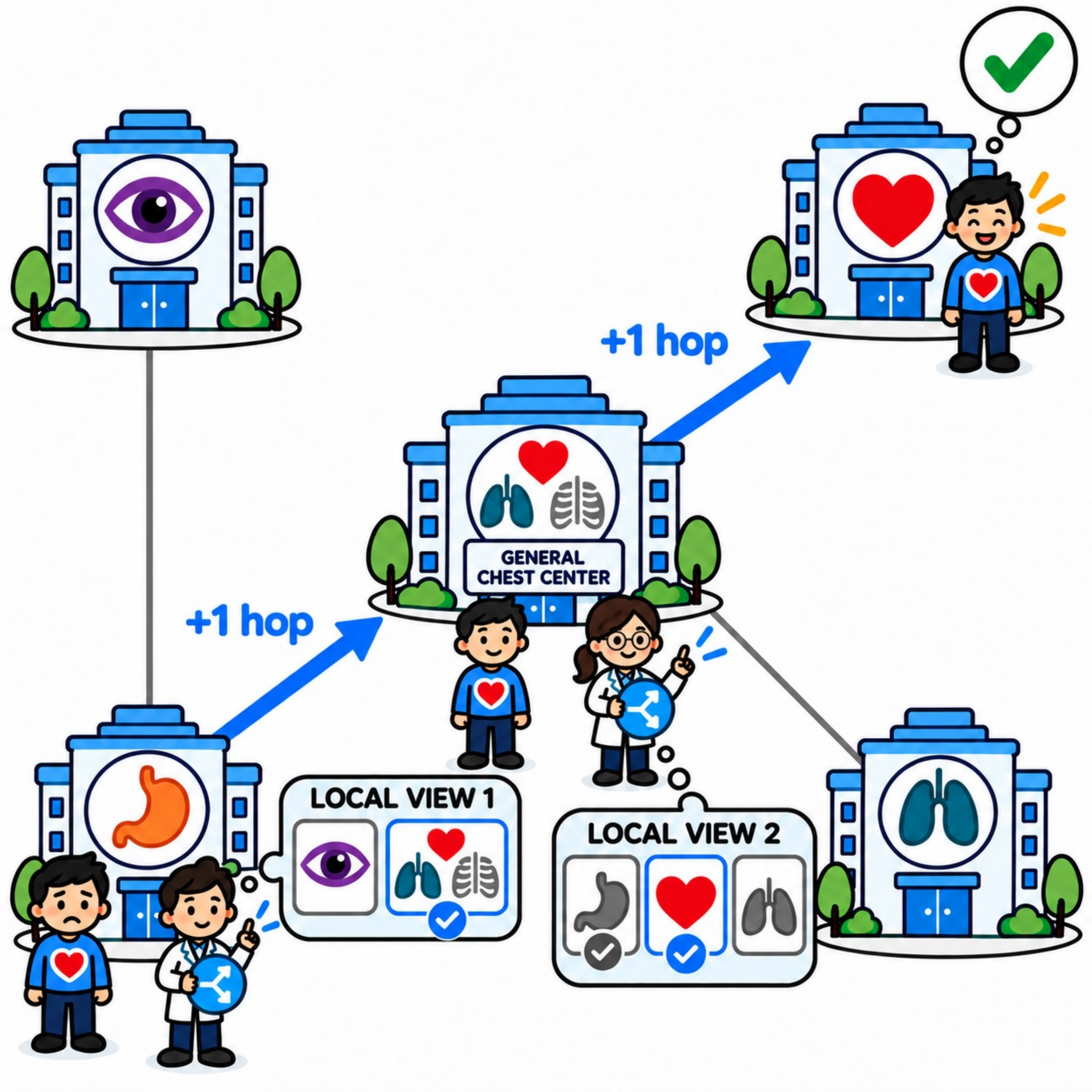}

    {\footnotesize (c) \textbf{Proposed local router.}}
  \end{minipage}
  \vskip -0.05in
  \caption{Hospital analogy for inference-time routing. 
  Gray lines indicate hospital pairs that can exchange information and know one another's capabilities, for example because they operate in the same region or health system. 
  Without routing, a patient with a heart condition remains at an entry hospital whose expertise lies in gastroenterology (left). 
  A centralized router sees all hospitals and dispatches the patient directly (center), but such global coordination is often infeasible because of administrative boundaries, privacy constraints, and geographic separation. 
  In contrast, \texttt{EdgeDRM} (ours) makes stepwise local referrals (right): the gastroenterology hospital first routes the patient to a general chest center, whose thoracic expertise better matches the condition. The center then identifies the condition as cardiac and routes the patient to a dedicated cardiology hospital. 
  Thus, each hop progressively improves the match between the patient and the hospital's expertise.}
  \label{fig:hospital-referral-overview}
  \vskip -0.1in
\end{figure*}

Existing work addresses only individual parts of this problem.
Personalized and decentralized FL methods determine how clients collaborate during training, but typically leave each test query at the client where it arrives~\citep{sun2023dfedavgm,zantedeschi2020fully,dai2022dispfl,lin2025fedspd,DBLP:conf/nips/LiWSLS25}.
Federated routing and mixture-of-experts methods exploit specialization through global client scores, learned routing gates, or expert subsets, but do not model budgeted neighbor-to-neighbor query movement~\citep{wang2026feddrm,reisser2021fedmix,farhat2023learning}.
Graph-based referral and search methods support local query movement, but learn from observed referral outcomes or graph-wide search supervision rather than from the static datasets held by neighboring FL clients~\citep{khudabukhsh2016distributed,baranchuk2019learning}.
Together, these lines leave open how to learn query--client compatibility from decentralized client data and compose neighbor-local comparisons into an effective prediction path.
See Appendix~\ref{app:related_work} for details.

Addressing this gap requires overcoming two coupled difficulties specific to decentralized routing.
First, query--client comparisons must be learned without the graph-wide normalization used by centralized routers such as \texttt{FedDRM}~\citep{wang2026feddrm}.
Second, neighbor-local comparisons must remain coherent when composed along a multi-hop referral path.
We therefore propose \texttt{EdgeDRM}, a decentralized density-ratio routing framework.
It avoids graph-wide normalization by learning each edge comparison through balanced endpoint source classification.
To make these comparisons composable, \texttt{EdgeDRM} assigns each client one routing potential shared across its incident edges.
Under a common representation, the resulting differences accumulate to the potential difference between the entry client and each reached client.
We also characterize discrepancies caused by different local representations.
At deployment, each client compares only its neighbors using its own query representation and stored neighbor coefficients, requiring neither a global score vector nor predictions from unvisited clients.
This work makes four contributions:
\begin{itemize}[leftmargin=*,topsep=0pt,parsep=0pt,itemsep=0pt]
  \item We formulate inference-time collaboration in decentralized FL as budgeted local expert referral, in which queries move only between neighbors and only visited clients contribute to prediction.
  \item We develop \texttt{EdgeDRM}, a neighbor-local framework that replaces graph-wide normalization with balanced endpoint source classification.
  One potential per client makes edge comparisons coherent across hops, and history-weighted aggregation uses only visited clients.
  \item We show that, under a shared representation, \texttt{EdgeDRM} learns increasingly accurate neighbor comparisons as training data and optimization steps increase.
  When locally better neighbors consistently guide the route, these comparisons yield cumulative gains in query--client compatibility and lead to the globally best-matched client given sufficient referral budget.
  \item Across benchmark, scaling, tiered-hospital, and preliminary real-world medical evaluations, \texttt{EdgeDRM} outperforms graph-matched decentralized baselines in most settings, with clear gains on sparse graphs but no consistent advantage under complete connectivity.
\end{itemize}

\section{Problem setting, distributional models, and objective}
\label{sec:edgedrm-model}

\textbf{Problem setting.}
Consider a $K$-class classification problem with $m$ clients indexed by $V=[m]=\{1,2,\ldots,m\}$.
The clients communicate over a fixed graph $G=(V,E)$.
An edge $\{i,j\}\in E$ allows clients $i$ and $j$ to exchange model information during training and refer queries directly to one another during inference.
During inference, each client can communicate only with its neighbors in $\mathcal N(i)$ and cannot access the models or scores of non-neighboring clients.
A query $x$ first arrives at an entry client $s$.
A \emph{referral} transfers the query from its current client to one of its neighbors, and successive referrals form a path in $G$.
The \emph{movement budget} $L\in\mathbb N_0$ is an upper bound on the number of such neighbor-to-neighbor transfers after the query arrives.
Thus, the query visits at most $L+1$ clients, including the entry client $s$.

Each client $i$ has a local dataset $D_i=\{(x_{ir},y_{ir})\}_{r=1}^{n_i}$ consisting of samples drawn independently from a client-specific joint distribution $P_i^{X,Y}$.
Let $N=\sum_{i=1}^m n_i$ and let $\rho_i=n_i/N$ denote client $i$'s sample proportion.
For each class $k\in[K]$, let $\pi_{ik}=P_i^Y(\{k\})$ denote the marginal probability of class $k$ at client $i$, write $\pi_i=(\pi_{i1},\ldots,\pi_{iK})$, and let $P_{ik}^X$ denote the conditional feature distribution given $Y=k$.
The feature marginal is then $P_i^X=\sum_{k=1}^K\pi_{ik}P_{ik}^X$.
We allow the clients' distributions to differ, so that $P_i^{X,Y}\neq P_j^{X,Y}$ for some $i\neq j$.

\textbf{Distributional models and conditional objective.}
The joint distribution factorizes as $P_{i}^{X,Y} = P_i^{Y\mid X}P_i^X$.
Let $g_\theta$ denote a shared encoder and let $Z=g_\theta(X)$ be the resulting representation.
We model $P_i^{Y\mid X}$ under the restriction that it depends on $X$ only through $Z$.
Then rather than modeling $P_i^X$ directly, we model its induced representation distribution $P_i^Z$.

\emph{\underline{Conditional label distribution}.}
For each client $i$, we model the conditional label distribution using the shared encoder $g_\theta$ and a private prediction head $\omega_i=(\alpha_i,\beta_i)$:
\begin{equation}
P_i^{Y\mid X}(\{k\}\mid x;\theta,\omega_i)
=\frac{\exp\{\alpha_{ik}+\beta_{ik}^{\top}g_\theta(x)\}}
{\sum_{c=1}^K\exp\{\alpha_{ic}+\beta_{ic}^{\top}g_\theta(x)\}},
\quad i\in V,~k\in[K].
\label{eq:local-predictive-model}
\end{equation}
For identifiability, we set $\sum_{k=1}^K\alpha_{ik}=0$ and $\sum_{k=1}^K\beta_{ik}=0$ for each client $i$.
The shared encoder provides a common representation across clients, while the client-specific heads allow their conditional label distributions to differ.

\emph{\underline{Representation marginal distribution}.}
To share information across clients, we jointly model the representation marginals $\{P_i^Z\}_{i=1}^m$ through a common-reference density-ratio model (DRM), following \texttt{FedDRM}~\citep{wang2026feddrm}.
Specifically, we assume that the measures $\{P_i^Z\}_{i=1}^m$ admit strictly positive Radon--Nikodym derivatives with respect to a common baseline distribution $F$ and satisfy:
\begin{equation}
\label{eq:common-reference-drm}
\varphi_i(z):=\{dP_i^Z/dF\}(z)=\exp\{b_i+\vartheta_i^\top z\},
\qquad i\in V,
\end{equation}
and $\varphi_i$ is the density of $P_i^Z$ relative to $F$ parameterized by $b_i$ and $\vartheta_i$.
For identifiability, we take $F=P_1^Z$ and set $b_1=0$ and $\vartheta_1=0$.
Because all clients share the same baseline distribution, it cancels from every pairwise comparison: $\log\{dP_j^Z/dP_i^Z\}(z)=\log\{\varphi_j(z)/\varphi_i(z)\}=(b_j-b_i)+(\vartheta_j-\vartheta_i)^\top z$.
Each client is associated with one coefficient pair $(b_i,\vartheta_i)$.
For a connected graph, differences across its edges determine these pairs relative to client 1.

\emph{\underline{From empirical likelihood to a conditional objective}.}
The baseline distribution $F$ is unknown, so an empirical-likelihood (EL) construction leaves it unrestricted and assigns masses to the distinct pooled representations $z_{ir}=g_\theta(x_{ir})$.
These masses must normalize both $F$ and every tilted distribution $\varphi_i\,dF$.
To state the connection with source classification, let $C$ denote the source client and define $P_\rho^{C,Z}$ by $P_\rho^C(\{i\})=\rho_i$ and $P_\rho^{Z\mid C}(\cdot\mid i)=P_i^Z$.
Under the normalized DRM, Bayes' rule gives:
\begin{equation*}
s_i(z)
:=P_\rho^{C\mid Z}(\{i\}\mid z)
=\frac{\rho_i\varphi_i(z)}{\sum_{c=1}^m\rho_c\varphi_c(z)}
=\frac{\exp\{\log\rho_i+b_i+\vartheta_i^\top z\}}{\sum_{c=1}^m\exp\{\log\rho_c+b_c+\vartheta_c^\top z\}}.
\end{equation*}
Thus, $s_i(z)$ is the posterior probability that a representation $z$ originates from client $i$.
For fixed representations and finite slopes $\vartheta$, maximizing EL over both the baseline masses and the free intercepts $b_2,\ldots,b_m$ is equivalent, up to a constant independent of these coefficients, to maximizing $\sum_{i,r}\log s_i(z_{ir})$ over the intercepts.
Appendix~\ref{app:edgedrm-derivation} proves this identity, including existence of the intercept optimum and the treatment of repeated representations.
Profiling the masses alone does not give a pointwise equality at arbitrary intercepts.
Motivated by this EL connection, we use the softmax expression as a conditional source model and combine it with the conditional-label model to define
\begin{equation}
\label{eq:loss}
\gL(\theta,\omega,b,\vartheta)=
\underbrace{
\left[-\frac{1}{N}\sum_{i=1}^m\sum_{r=1}^{n_i}
\log P_i^{Y\mid X}(\{y_{ir}\}\mid x_{ir};\theta,\omega_i)\right]
}_{\text{conditional-label loss}}
+\underbrace{
\left[-\frac{1}{N}\sum_{i=1}^m\sum_{r=1}^{n_i}\log s_i(z_{ir})\right]
}_{\text{source-classification loss}}.
\end{equation}
The conditional-label loss fits client-specific predictors, while the source-classification loss estimates density ratios by identifying each representation's source.
For an edge $\{i,j\}$, the odds $s_j(z)/s_i(z)$ provide the comparison used to evaluate referral from client $i$ to client $j$.
These odds measure distributional compatibility, not prediction accuracy; predictions come from the conditional-label models at visited clients.
Computing the global source-classification loss requires all client scores and a graph-wide normalization, which graph-local communication cannot provide.
This limitation distinguishes decentralized FL from \texttt{FedDRM} and motivates a conditional source comparison on each communication edge.

\section{Decentralized Training and Local Expert Referral}
\label{sec:edgedrm-optimization}
Section~\ref{sec:edgedrm-model} defines the statistical model and its global objective.
We now describe how clients train this model through neighbor communication and use the fitted quantities for referral and prediction.
\subsection{Decentralized training with graph-local source likelihoods}
\label{sec:decentralized-training}
Because the two terms in~\eqref{eq:loss} impose different communication requirements, we train the model in two stages.
Stage I trains the encoders and private prediction heads using the conditional-label loss and neighbor mixing.
Stage II freezes the learned encoders, refines the private prediction heads, and estimates the density-ratio coefficients $(b_i,\vartheta_i)$ from graph-local source likelihoods.

For training, client $i$ forms the local objective $\mathcal L_i$ in~\eqref{eq:edgedrm-local-objective} from its local conditional-label loss (the first term in~\eqref{eq:loss}) and one local source-classification loss for each incident edge, all evaluated on $D_i$.
The graph objective in~\eqref{eq:edgedrm-graph-objective} then averages these client-local objectives with equal weights: $\mathcal L_G=m^{-1}\sum_i\mathcal L_i$.
For the conditional-label component, this replaces the sample-proportion weight $\rho_i=n_i/N$ in~\eqref{eq:loss} by $1/m$.
This equal weighting is the canonical formulation in decentralized optimization and is widely used in decentralized FL, where doubly stochastic neighbor mixing gives each client equal weight~\citep{lalitha2019peer,koloskova2020unified,sun2023dfedavgm}.
This weighting convention differs from the centralized setting. 

\textbf{Stage I: conditional-label model training.}
The model in Section~\ref{sec:edgedrm-model} uses a shared encoder parameter $\theta$, but decentralized optimization requires each client to maintain a local copy $\theta_i$.
For compactness, let $\psi_i=(\theta_i,\omega_i)$ collect the local encoder and prediction-head parameters at client $i$.
Stage I runs for $R_{\rm sh}$ communication rounds.
At round $r$, client $i$ starts from $\psi_i^{(r)}$ and takes $B_{\rm sh}$ SGD steps on its empirical conditional-label loss using $D_i$.
Let $\widetilde\psi_i^{(r)}$ denote the resulting encoder and prediction-head parameters.
The graph $G$ determines which pairs of clients can exchange these intermediate parameters.
Because no central server aggregates all client updates, each client instead combines its own intermediate parameters with those received from its neighbors.
A mixing matrix $W$ determines the weights in this local aggregation.
Specifically, let $W$ be symmetric and doubly stochastic, with $W_{ij}=0$ whenever $j\notin\{i\}\cup\mathcal N(i)$.
Client $i$ then updates its parameters as
\begin{equation}
\psi_i^{(r+1)}
=\sum_{j\in\{i\}\cup\mathcal N(i)}W_{ij}\widetilde\psi_j^{(r)}.
\label{eq:generic-decentralized-update}
\end{equation}
For each edge $\{i,j\}\in E$, we use the Metropolis weight $W_{ij}=1/(1+\max\{\deg(i),\deg(j)\})$ and set $W_{ii}=1-\sum_{j\ne i}W_{ij}$.
Neighbor mixing is an established decentralized optimization mechanism~\citep{koloskova2020unified}, but finitely many rounds need not produce exact consensus, so the learned encoders may differ across clients.

\textbf{Stage II: frozen-encoder density-ratio fitting.}
Stage I uses local SGD and neighbor mixing to optimize the client-averaged conditional-label objective described above.
Stage II addresses the remaining source-classification loss, which learns the density ratios by predicting the source client $C$ from a representation $Z$.
In centralized FL, a server can coordinate the scores of all $m$ clients required to evaluate the source probabilities $s_i(z)$.
In the decentralized setting, however, each client communicates only with its neighbors and therefore cannot evaluate the graph-wide normalization in $s_i(z)$.
Because each referral traverses one edge, deciding whether to move the query requires only a comparison between the two endpoints.
We therefore replace the global $m$-client source-classification problem with a binary source-classification problem on each edge.

For each edge $e=\{i,j\}\in E$ with $i<j$, we form a balanced binary source-classification problem by assigning equal weight to the two endpoint mean losses.
Its posterior probability that $z$ comes from client $j$ rather than client $i$ is $q_{ij}^{\rm bal}(z):=\varphi_j(z)/\{\varphi_i(z)+\varphi_j(z)\}$, with
\begin{equation}
\label{eq:logitqij}
\operatorname{logit}q_{ij}^{\rm bal}(z)=(b_j-b_i)+(\vartheta_j-\vartheta_i)^\top z.
\end{equation}
This binary problem is the source-classification term in~\eqref{eq:loss} restricted to the endpoints $\{i,j\}$, and it learns the coefficient differences needed to compare them.
Because client $i$ uses the same coefficient pair $(b_i,\vartheta_i)$ on all its incident edges, the graph links these pairwise comparisons through their shared endpoints.
The two endpoint losses receive equal weight, irrespective of $n_i$ and $n_j$. Thus, the sample-count prior is not included in the training logit.
The exact graph-local objective is given in Appendix~\ref{app:weighted-graph-objective}.
After Stage I, clients freeze their encoders and fit their private heads and density-ratio coefficients using local conditional-label and incident-edge source losses.
Only sample counts, density-ratio coefficients, and their gradient contributions are exchanged across edges.
Encoders, private heads, and raw data remain local.
These fitted quantities drive referral and prediction.

\subsection{Budgeted local expert referral \& Prediction using the referral history}
\label{sec:referral-history}
\textbf{Budgeted local expert referral.}
For a query $x$, let $\Delta_{ij}(x):=\log(\rho_j/\rho_i)+\operatorname{logit}q_{ij}^{\rm bal}(g_\theta(x))$ denote the population pairwise log-odds obtained by adding the sample-count prior to the learned balanced comparison.
This quantity provides the criterion for referring the query from client $i$ to client $j$.
At inference, client $i$ substitutes the fitted density-ratio coefficients and its own fitted representation into this quantity, giving the plug-in estimator
\begin{equation}
\widehat\Delta_{ij}(x)
:=\log(\rho_j/\rho_i)+(\widehat b_j-\widehat b_i)
+(\widehat\vartheta_j-\widehat\vartheta_i)^\top g_{\widehat\theta_i}(x).
\label{eq:fitted-edge-score}
\end{equation}
Hence training estimates the density-ratio term without sample-count weighting, and inference adds $\log(\rho_j/\rho_i)=\log(n_j/n_i)$ exactly once.
Applying these fitted endpoint comparisons sequentially produces a query-dependent route through the graph.
For a query $x$ entering at client $s$ under movement budget $L$, the realized route is
\begin{equation}
\label{eq:budgeted-route}
\mathcal R_L(x,s)=(\widehat s_0=s,\widehat s_1,\ldots,\widehat s_T),
\quad T\le L,
\quad \{\widehat s_{t-1},\widehat s_t\}\in E.
\end{equation}
\begin{wrapfigure}{r}{0.48\textwidth}
\vspace{-\baselineskip}
\centering
\footnotesize
\begin{minipage}{\linewidth}
\hrule
\vspace{0.4ex}
\refstepcounter{algorithm}
\textbf{Algorithm~\thealgorithm: Budgeted local expert referral.}
\label{alg:referral}
\par\vspace{0.4ex}
\hrule
\begin{algorithmic}[1]
\STATE {\bfseries Input:} Query $x$, entry client $s$, threshold $\eta\ge0$, and movement budget $L$.
\STATE $\widehat s_0\leftarrow s$, $V_{\rm vis}\leftarrow\{s\}$, $T\leftarrow0$
\WHILE{$T<L$}
  \STATE $U\leftarrow\mathcal N(\widehat s_T)\setminus V_{\rm vis}$
  \STATE \textbf{if} $U=\varnothing$ \textbf{then break}
  \STATE $j^\star\leftarrow\arg\max_{j\in U}\widehat\Delta_{\widehat s_T,j}(x)$
  \STATE \textbf{if} $\widehat\Delta_{\widehat s_T,j^\star}(x)\le\eta$ \textbf{then break}
  \STATE Send the packet and accepted score to $j^\star$
  \STATE $T\leftarrow T+1$, $\widehat s_T\leftarrow j^\star$
  \STATE $V_{\rm vis}\leftarrow V_{\rm vis}\cup\{j^\star\}$
\ENDWHILE
\STATE {\bfseries Output:} $(\widehat s_0,\ldots,\widehat s_T)$ and the accepted scores.
\end{algorithmic}
\hrule
\end{minipage}
\vspace{-1.5\baselineskip}
\end{wrapfigure}
The route construction is detailed in Algorithm~\ref{alg:referral}.
We describe its intuition below.
During referral, $x$ is passed between neighboring clients together with the visited-client set, accepted edge scores, and remaining movement budget.
At client $\widehat s_t$, the router scores each unvisited neighbor using~\eqref{eq:fitted-edge-score} and selects the largest score, with ties broken deterministically.
The query moves to the selected neighbor only if this score exceeds $\eta$ and otherwise terminates.
Evaluating a candidate requires only the current client's representation of $x$ and the stored sample counts and density-ratio coefficients of the two endpoints.
The candidate need not inspect the query.
The procedure uses neither the query label nor information about prediction correctness.
It returns the visited route and accepted edge scores for the final prediction rule.

\textbf{Prediction using the referral history.}
Because referral scores reflect distributional compatibility rather than prediction accuracy directly, earlier visited clients may still provide useful complementary predictions.
We therefore aggregate the predictions from all clients on $\mathcal R_L(x,s)$.
Denote the fitted probability vector at client $i$ as $f_i(x):=\left( P_i^{Y\mid X}(\{k\}\mid x;\widehat\theta_i,\widehat\omega_i) \right)_{k=1}^K$.
For each visited client $\widehat s_t$, let $u_0(x,s)=0$ and $u_t(x,s)=\sum_{\ell=1}^t\widehat\Delta_{\widehat s_{\ell-1},\widehat s_\ell}(x)$ denote the cumulative accepted edge score from the entry client to $\widehat s_t$ for $t=1,\ldots,T$.
We set $(w_0(x,s),\ldots,w_T(x,s)):=\operatorname{softmax}(u_0(x,s),\ldots,u_T(x,s))$.
The final prediction is:
\begin{equation*}
f_{\rm mix}(x,s)
=\sum_{t=0}^T w_t(x,s)f_{\widehat s_t}(x),
\qquad
\widehat y=\arg\max_{k\in[K]}f_{{\rm mix},k}(x,s).
\end{equation*}
\section{Theoretical Results}
\label{sec:theory}

We establish theoretical guarantees for coefficient estimation and referral gain in \texttt{EdgeDRM}.
Our results characterize the accuracy and asymptotic normality of parameter estimation (Section~\ref{sec:score-error-decomposition}) and referral gains (Section~\ref{sec:structural-gain-settings}).
Proofs and additional results are given in Appendix~\ref{app:proofs}.

\subsection{How accurately can we estimate the parameters?}
\label{sec:score-error-decomposition}

We first study the performance of parameter estimation in stage II.
To separate coefficient estimation from representation learning, consider a common representation $g_\theta$ from Section~\ref{sec:edgedrm-model}, fixed independently of the coefficient-fitting samples.
Let $\zeta$ collect the free density-ratio coefficients $(\vartheta_i,b_i)$, $i\ne1$, and private-head parameters $(\alpha_i,\beta_i)$ under the normalizations in Section~\ref{sec:edgedrm-model}.

We analyze $R$ updates of Stage-II optimization on $\mathcal L_G$ in \eqref{eq:edgedrm-graph-objective}, holding $(\vartheta_1,b_1)=0$ fixed and using unbiased mini-batch gradients with step sizes $\eta_r=2/\{\mu(r+r_0)\}$ for $r=0,\ldots,R-1$.
Here $0<\mu\le1$ is a curvature constant and $r_0>0$ is a fixed offset.
We refer to Appendix~\ref{app:parameter-proof} for further details on conditions on these constants.
The following theorem bounds the statistical and optimization errors of the Stage-II output and establishes its asymptotic normality.

\begin{theorem}
\label{thm:parameter-estimation}
Let $\zeta^\star$ denote the true coefficient vector of the models in Section~\ref{sec:edgedrm-model}, and suppose $n_i/N\to\rho_i^*>0$.
Assume the coefficients are identifiable and assume the population Hessian at $\zeta^\star$ is positive definite and, with probability tending to one, $\mu I\preceq\nabla^2\mathcal L_G(\zeta)\preceq\mu^{-1}I$,
throughout a convex region containing a finite empirical minimizer and the optimization trajectory, for a constant $0<\mu\le1$ independent of $N$.
Under the remaining regularity conditions in Appendix~\ref{app:parameter-proof}, the Stage-II output $\zeta^{(R)}$ from the specified SGD procedure satisfies, as $N,R\to\infty$,
\begin{equation}
\|\zeta^{(R)}-\zeta^\star\|_2
=O_p\!\left(N^{-1/2}+R^{-1/2}\right).
\label{eq:coefficient-total-error}
\end{equation}
If, in addition, $R/N\to\infty$, then for some $\Sigma_G$ specified in Appendix~\ref{app:parameter-proof},
\begin{equation}
\sqrt N\{\zeta^{(R)}-\zeta^\star\}
\xrightarrow{d}\mathcal N(0,\Sigma_G),
\label{eq:algorithm-clt}
\end{equation}
\end{theorem}

The first term in \eqref{eq:coefficient-total-error} is the statistical error from $N$ training samples.
The second is the optimization error after $R$ updates.
Thus, when $R/N\to\infty$, the optimization error is dominated by $N^{-1/2}$, and the output has the asymptotic distribution in \eqref{eq:algorithm-clt}.

We note that, under the stated models, changing client sample proportions does not change the population minimizer $\zeta^\star$.
The proportions need only satisfy $n_i/N\to\rho_i^*>0$, without a specified convergence rate.
The coefficient-estimation bound in Theorem~\ref{thm:parameter-estimation} implies a bound on the edge-comparison errors used for referral.
In particular, for any fixed $x$ with finite $g_\theta(x)$, we obtain
\begin{equation}
\max_{\{i,j\}\in E}|\widehat\Delta_{ij}(x)-\Delta_{ij}(x)|
=O_p\!\left(N^{-1/2}+R^{-1/2}\right).
\label{eq:comparison-error-bound}
\end{equation}
This comparison-error bound yields the referral guarantee in Section~\ref{sec:structural-gain-settings}.
We remark that the result applies to coefficient fitting with a common, independently fixed encoder.

\subsection{How far and how much can referral improve?}
\label{sec:structural-gain-settings}

Fix a query $x$ and entry client $s$, and consider the route $\mathcal R_L(x,s)$ in~\eqref{eq:budgeted-route}.
Using the population coefficients of Section~\ref{sec:edgedrm-model}, write $h_i(x)=\log\rho_i+b_i+\vartheta_i^\top g_\theta(x)$ and $h^\star(x)=\max_{i\in V}h_i(x)$.
Thus, $\Delta_{ij}(x)=h_j(x)-h_i(x)$, and $\Gamma_L(x,s)=h_{\widehat s_T}(x)-h_s(x)$ is the referral gain.
Under the DRM, this gain is the log ratio of the terminal and entry clients' prior-weighted densities.
Our referral guarantees rely on the following intuitive structural properties.

\textbf{Minimum local improvement.}
Assume there is $\gamma>0$ such that every nonoptimal client has a neighbor whose potential is at least $\gamma$ higher:
\begin{equation}
\max_{j\in\mathcal N(i)}\{h_j(x)-h_i(x)\}\ge\gamma
\quad\text{whenever }h_i(x)<h^\star(x).
\label{eq:minimum-ascent}
\end{equation}
Thus, a client with no higher-potential neighbor must be globally optimal.

\textbf{Ascent depth.}
For $\epsilon\ge0$, let $H_\epsilon$ be the longest strictly ascending path length from $s$, with each step choosing a neighbor within $2\epsilon$ of the highest neighbor potential.
Then $H_\epsilon\le m-1$, and under the conditions below, it bounds the number of accepted referrals.

\begin{theorem}
\label{thm:structural-gains}
Under the conditions of Theorem~\ref{thm:parameter-estimation}, use $\zeta^{(R)}$ for referral at a fixed query $x$ with finite $g_\theta(x)$.
For any $\delta\in(0,1)$, there exists a constant $C_\delta>0$, independent of $N$ and $R$, such that the following holds with $\epsilon=C_\delta(N^{-1/2}+R^{-1/2})$.
Suppose \eqref{eq:minimum-ascent} holds, $\eta\ge\epsilon$, and $\gamma>\eta+\epsilon$.
For all sufficiently large $N,R$, with probability at least $1-\delta$, Algorithm~\ref{alg:referral} with movement budget $L\in\mathbb N_0$ makes $T\le\min\{L,H_\epsilon\}$ moves and satisfies
\begin{equation}
\Gamma_L(x,s)\ge(\gamma-2\epsilon)\min\{L,H_\epsilon\}.
\label{eq:unified-referral-gain}
\end{equation}
On the same event, $L\ge H_\epsilon$ guarantees that the route reaches a maximizer.
\end{theorem}

The guarantee in \eqref{eq:unified-referral-gain} lower-bounds $\Gamma_L(x,s)$, the improvement in source compatibility, by a product of two terms.
The first term, $\gamma-2\epsilon$, lower-bounds the population gain per accepted move.
Larger per-step gains allow substantial improvement with fewer referrals.
The second term, $\min\{L,H_\epsilon\}$, is the effective step count.
A smaller ascent depth $H_\epsilon$ requires less budget to guarantee reaching a global maximizer, while with $L<H_\epsilon$ the route may exhaust its budget first.
Note that improvement in source compatibility does not by itself imply improved predictive performance.
Appendix~\ref{app:structural-proof} provides the proof and analyzes ascent depth and referral gain in several common graph types.

\section{Experiments}
\label{sec:experiments}

\subsection{Experimental setup and evaluation}
\textbf{Benchmarks and topologies.}
We first evaluate on CIFAR-10, CIFAR-20, and CIFAR-100 benchmarks under 2 non-IID partitions.
In the shards-per-client (SPC) partition based on~\citep{mcmahan2017communication}, each client receives five equal-sized, label-homogeneous shards for CIFAR-10 and 25 for CIFAR-20/100.
In the Dirichlet partition, each client's class proportions are drawn from a Dirichlet distribution with concentration $0.3$.
The main benchmark has eight clients.
We consider four communication topologies: clients are connected as a \textsc{Path}, a \textsc{Cycle}, a connected random graph in which every client has four neighbors (\textsc{RR4}), or a \textsc{Complete} graph.
These are standard decentralized communication topologies and are ordered by increasing connectivity.
For further details on the benchmarks, graph topologies, and data construction, see Appendix~\ref{app:frozen-data}.

\textbf{Training and inference.}
We train randomly initialized ResNet-18 models for 800 communication rounds using \texttt{DFedAvgM}~\citep{sun2023dfedavgm}, freeze each local encoder, and then fit the personalized classifier heads and routers for 200 rounds.
The routing objective weights the two endpoints of every edge equally.
At inference, \texttt{EdgeDRM} starts from a uniformly sampled non-source entry, excludes revisits, and averages predictions from visited experts using temperature-one score weights.

\textbf{Comparisons and reporting.}
At test time, each query from source client $s$ is assigned uniformly at random to an entry client other than $s$.
\texttt{Local} trains one model per client using only local data and predicts with the model at that entry.
The decentralized baselines \texttt{DFedAvgM}~\citep{sun2023dfedavgm}, \texttt{DisPFL}~\citep{dai2022dispfl}, and \texttt{FedSPD}~\citep{lin2025fedspd} also predict with their learned model at that entry.
By contrast, we include centralized \texttt{FedDRM}~\citep{wang2026feddrm} only as an illustrative reference because its client-selection head scores all clients globally and therefore cannot operate under the neighbor-only communication constraints of our decentralized setting.
\texttt{EdgeDRM-NR} uses the same trained model but keeps each query at its entry client.

\subsection{Main results}
\begin{table}[!htbp]
\centering
\vskip -0.05in
\caption{\textbf{Eight-client decentralized CIFAR benchmark} reporting system accuracy (\%) across three datasets and two non-IID partitions.
\texttt{Local} and centralized \texttt{FedDRM} are illustrative references.
Bold cells with blue shading mark the best decentralized result within each topology and setting.
All entries are means $\pm$ sample SD across 3 seeds.}
\label{tab:decentralized-comparison}
\definecolor{resultblue}{RGB}{40,112,181}
\renewcommand{\arraystretch}{0.96}
\setlength{\tabcolsep}{2.4pt}
\resizebox{\linewidth}{!}{%
\begin{tabular}{@{}llrrrrrr@{}}
\toprule
\multirow{2}{*}{Group}&\multirow{2}{*}{Method} & \multicolumn{2}{c}{CIFAR-10} & \multicolumn{2}{c}{CIFAR-20} & \multicolumn{2}{c}{CIFAR-100} \\
\cmidrule(lr){3-4}\cmidrule(lr){5-6}\cmidrule(lr){7-8}
 &  & Dir-0.3 & 5-SPC & Dir-0.3 & 25-SPC & Dir-0.3 & 25-SPC \\
\midrule
\multirow{2}{*}{\textbf{Reference}}
& \texttt{Local} & $16.59\pm3.78$ & $18.46\pm1.49$ & $13.35\pm0.29$ & $11.17\pm0.54$ & $4.03\pm0.23$ & $2.49\pm0.10$ \\
& \texttt{FedDRM} & $63.81\pm1.54$ & $61.24\pm2.29$ & $39.28\pm0.87$ & $42.42\pm1.12$ & $28.67\pm0.58$ & $33.38\pm0.87$ \\
\midrule
\multirow{5}{*}{\textbf{\textsc{Path}}}
& \texttt{DFedAvgM} & $28.55\pm7.62$ & $27.41\pm0.21$ & $23.94\pm0.03$ & $21.47\pm1.04$ & $9.28\pm0.61$ & $7.03\pm0.44$ \\
& \texttt{DisPFL} & $27.19\pm8.01$ & $26.87\pm1.91$ & \cellcolor{resultblue!15}$\mathbf{26.43\pm0.65}$ & $22.79\pm1.68$ & $10.83\pm0.96$ & $6.41\pm0.64$ \\
& \texttt{FedSPD} & $18.34\pm5.17$ & $20.00\pm1.14$ & $15.44\pm0.99$ & $14.41\pm0.78$ & $4.68\pm0.38$ & $4.30\pm0.11$ \\
& \texttt{EdgeDRM-NR} & $22.00\pm4.85$ & $21.53\pm1.18$ & $21.57\pm0.35$ & $18.51\pm1.79$ & $6.89\pm0.62$ & $4.21\pm0.21$ \\
& \texttt{EdgeDRM} & \cellcolor{resultblue!15}$\mathbf{33.08\pm3.53}$ & \cellcolor{resultblue!15}$\mathbf{32.88\pm1.00}$ & $24.87\pm0.33$ & \cellcolor{resultblue!15}$\mathbf{23.97\pm1.12}$ & \cellcolor{resultblue!15}$\mathbf{11.19\pm0.58}$ & \cellcolor{resultblue!15}$\mathbf{11.15\pm0.34}$ \\
\midrule
\multirow{5}{*}{\textbf{\textsc{Cycle}}}
& \texttt{DFedAvgM} & $30.34\pm6.23$ & $27.63\pm2.22$ & $25.00\pm0.15$ & $23.19\pm1.11$ & $10.39\pm0.72$ & $7.93\pm0.45$ \\
& \texttt{DisPFL} & $29.10\pm7.59$ & $28.02\pm1.12$ & \cellcolor{resultblue!15}$\mathbf{27.54\pm0.38}$ & $24.49\pm2.06$ & $11.42\pm0.50$ & $6.99\pm0.61$ \\
& \texttt{FedSPD} & $20.05\pm4.01$ & $20.95\pm1.13$ & $15.39\pm0.65$ & $14.78\pm1.04$ & $5.02\pm0.20$ & $4.74\pm0.41$ \\
& \texttt{EdgeDRM-NR} & $22.92\pm4.71$ & $21.95\pm0.75$ & $21.99\pm0.22$ & $19.04\pm1.62$ & $7.49\pm0.71$ & $4.32\pm0.39$ \\
& \texttt{EdgeDRM} & \cellcolor{resultblue!15}$\mathbf{35.20\pm2.30}$ & \cellcolor{resultblue!15}$\mathbf{34.79\pm0.63}$ & $25.99\pm0.59$ & \cellcolor{resultblue!15}$\mathbf{25.68\pm0.94}$ & \cellcolor{resultblue!15}$\mathbf{12.40\pm0.63}$ & \cellcolor{resultblue!15}$\mathbf{12.08\pm0.42}$ \\
\midrule
\multirow{5}{*}{\textbf{\textsc{RR4}}}
& \texttt{DFedAvgM} & $40.10\pm1.86$ & $37.58\pm1.05$ & $28.83\pm0.73$ & $28.14\pm1.41$ & $13.92\pm0.90$ & $12.14\pm0.54$ \\
& \texttt{DisPFL} & $30.31\pm7.77$ & $30.75\pm1.45$ & $28.89\pm1.24$ & $26.07\pm2.60$ & $12.64\pm0.78$ & $7.50\pm0.62$ \\
& \texttt{FedSPD} & $19.50\pm4.91$ & $21.69\pm0.93$ & $15.85\pm1.83$ & $14.44\pm1.57$ & $5.41\pm0.52$ & $4.49\pm0.22$ \\
& \texttt{EdgeDRM-NR} & $25.35\pm4.54$ & $23.44\pm1.23$ & $22.72\pm0.60$ & $20.09\pm1.81$ & $8.05\pm0.40$ & $5.53\pm0.18$ \\
& \texttt{EdgeDRM} & \cellcolor{resultblue!15}$\mathbf{50.12\pm2.21}$ & \cellcolor{resultblue!15}$\mathbf{46.19\pm1.17}$ & \cellcolor{resultblue!15}$\mathbf{29.91\pm0.79}$ & \cellcolor{resultblue!15}$\mathbf{31.51\pm0.52}$ & \cellcolor{resultblue!15}$\mathbf{17.59\pm0.59}$ & \cellcolor{resultblue!15}$\mathbf{20.01\pm0.52}$ \\
\midrule
\multirow{5}{*}{\textbf{\textsc{Complete}}}
& \texttt{DFedAvgM} & $55.04\pm1.47$ & \cellcolor{resultblue!15}$\mathbf{52.91\pm1.43}$ & \cellcolor{resultblue!15}$\mathbf{32.62\pm0.59}$ & \cellcolor{resultblue!15}$\mathbf{36.10\pm0.38}$ & \cellcolor{resultblue!15}$\mathbf{21.02\pm0.48}$ & \cellcolor{resultblue!15}$\mathbf{23.81\pm1.08}$ \\
& \texttt{DisPFL} & $36.63\pm5.08$ & $34.35\pm0.88$ & $30.10\pm0.99$ & $28.60\pm1.98$ & $15.09\pm0.97$ & $10.88\pm0.55$ \\
& \texttt{FedSPD} & $27.33\pm7.01$ & $27.29\pm3.58$ & $15.35\pm1.19$ & $15.21\pm0.58$ & $6.01\pm0.60$ & $6.14\pm0.38$ \\
& \texttt{EdgeDRM-NR} & $27.85\pm4.03$ & $25.80\pm0.93$ & $23.64\pm0.46$ & $21.12\pm1.44$ & $9.31\pm0.37$ & $6.88\pm0.23$ \\
& \texttt{EdgeDRM} & \cellcolor{resultblue!15}$\mathbf{56.62\pm2.87}$ & $52.84\pm1.86$ & $31.31\pm1.02$ & $34.12\pm0.31$ & $20.02\pm0.63$ & $23.65\pm0.67$ \\
\bottomrule
\end{tabular}%
}
\vskip -0.15in
\end{table}

\textbf{CIFAR benchmark results.}
The main results in Table~\ref{tab:decentralized-comparison} yield three main observations.
First, \texttt{EdgeDRM} is the strongest graph-constrained method in most settings.
It ranks first in 16 of the 18 dataset--partition configurations on the three sparse graphs and in 17 of the 24 graph--configuration pairs overall.
Its mean accuracy difference relative to the strongest competing decentralized method in each sparse-graph setting is $+3.61$ percentage points.
The two sparse-graph exceptions both occur on CIFAR-20/Dirichlet-$0.3$, which exhibits milder client label skew than CIFAR-20/25-SPC in the realized partitions.
Across the sparse graphs, its mean margin over the strongest competing decentralized method is $4.75$ percentage points under SPC versus $2.48$ under Dirichlet-$0.3$, a difference of $2.27$ points.
Taken together with the \texttt{EdgeDRM-NR} ablation below, the results suggest that learned referral helps exploit stronger client specialization.
Second, the \textsc{Complete} graph, equivalent to the central server case, \texttt{EdgeDRM} is effectively tied with \texttt{DFedAvgM}, averaging only $0.49$ points lower across settings.
So full graph connectivity narrows the advantage of \texttt{EdgeDRM}.
Both methods nevertheless remain below centralized \texttt{FedDRM} in all six settings, yielding the gap between routing designed for centralized setting. 
Third, referral itself provides a large and universal gain.
\texttt{EdgeDRM} outperforms \texttt{EdgeDRM-NR} in all configurations, by $3.30$--$28.77$ points and $11.87$ points on average.

\begin{wrapfigure}{r}{0.5\textwidth}
\centering
\vspace{-1\baselineskip}
\includegraphics[width=1.03\linewidth]{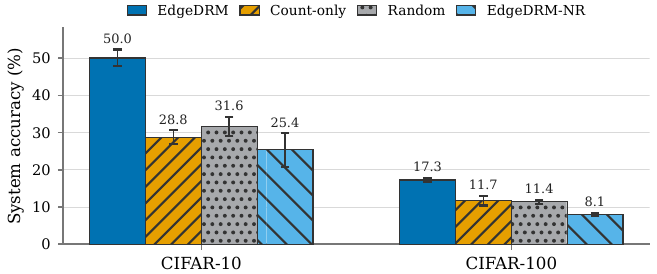}
\vspace{-\baselineskip}
\caption{\textbf{Fixed-expert hard-terminal routing ablation} on CIFAR-10 and CIFAR-100 with Dirichlet-$0.3$ partitions and \textsc{RR4}.
}
\label{fig:referral-policy-ablation}
\vspace{-\baselineskip}
\end{wrapfigure}
\textbf{Where the referral gain comes from.}
We conduct an ablation to isolate the effect of the routing rule.
Every method uses the terminal expert's hard prediction, without averaging predictions across visited experts.
Besides \texttt{EdgeDRM} and \texttt{EdgeDRM-NR}, we also compare \texttt{Count-only}, which triggers a referral when $\log(n_j/n_i)>0$, and \texttt{Random}, which samples uniformly among unvisited neighbors.
The improvement over \texttt{EdgeDRM-NR} and \texttt{Random} shows that selecting suitable neighbors is more effective than null/arbitrary referral.
Comparison with \texttt{Count-only} further shows that the gain comes from the trained router rather than solely from differences in local sample sizes.

\textbf{Scaling in number of clients.}
We test whether the behavior observed with few clients persists as the client population grows substantially.
Table~\ref{tab:thirty-two-experts} shows that \texttt{EdgeDRM} remains the strongest decentralized method in both settings.
We note that, because the total CIFAR-10 data pool is fixed, increasing the number of clients reduces the average local sample sizes, which likely contributes to the lower accuracy relative to the results in Table~\ref{tab:decentralized-comparison}.

\begin{table}[!htbp]
\centering
\caption{CIFAR-10/\textsc{RR4} system accuracy with 32 clients (\%; mean $\pm$ SD over three seeds).}
\label{tab:thirty-two-experts}
\resizebox{0.9\linewidth}{!}{
\begin{tabular}{@{}lc|ccccc|c@{}}
\toprule
& Reference & \multicolumn{5}{c|}{Decentralized} & Reference \\
\cmidrule(lr){2-2}\cmidrule(lr){3-7}\cmidrule(lr){8-8}
Setting & \texttt{Local} & \texttt{DFedAvgM} & \texttt{DisPFL} & \texttt{FedSPD} & \texttt{EdgeDRM-NR} & \texttt{EdgeDRM} & \texttt{FedDRM} \\
\midrule
Dir-0.3 & $16.45\pm1.15$ & $27.79\pm3.25$ & $23.71\pm2.65$ & $17.82\pm0.69$ & $21.52\pm1.50$ & $\mathbf{30.06\pm1.16}$ & $38.45\pm0.52$ \\
5-SPC & $18.11\pm0.47$ & $27.92\pm1.22$ & $24.87\pm0.95$ & $20.54\pm0.75$ & $22.76\pm1.20$ & $\mathbf{31.70\pm0.39}$ & $35.96\pm0.52$ \\
\bottomrule
\end{tabular}}
\vskip -0.1in
\end{table}

\subsection{Referral traffic}
\begin{wrapfigure}{r}{0.45\textwidth}
    \vspace{-3\baselineskip}
    \centering
    \includegraphics[width=\linewidth]{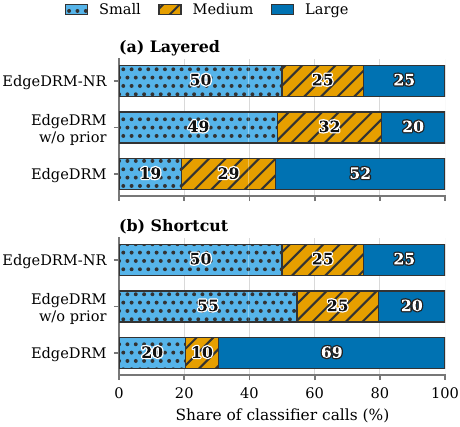}
    \vskip -0.05in
    \caption{\textbf{Referral traffic by hospital tier on two topologies}.
    Bars show the mean share of visited-classifier calls over 3 seeds.}
    \label{fig:hospital-traffic}
    \vskip -0.15in
\end{wrapfigure}
We conduct an ablation study to examine how a sample-count prior changes referral traffic across hospital tiers.
We use 32 clients: 16 small, eight medium, and eight large hospitals holding 250, 1000, and 4000 training examples, respectively.
We evaluate two graph structures.
In the layered graph, each tier forms a cycle, every small hospital connects to one medium hospital, and medium and large hospitals are matched one-to-one.
Thus, a small hospital can reach the large tier only through the medium tier, requiring at least two hops.
The shortcut graph retains all edges and adds one direct small--large edge for each small hospital, making a large hospital reachable from every small hospital in one hop.
\texttt{EdgeDRM w/o prior} applies the same history-weighted referral rule after removing the sample-count offset $\log(n_j/n_i)$ and serves as a sample-size-agnostic comparison.
In Figure \ref{fig:hospital-traffic}, \texttt{EdgeDRM w/o prior} largely preserves the \texttt{EdgeDRM-NR} tier proportions, whereas \texttt{EdgeDRM} substantially increases large-hospital traffic, with an even stronger concentration when shortcut edges provide direct access to large hospitals.
This pattern captures a realistic healthcare trade-off: the sample-count prior is a plausible way to represent the stronger expertise of larger hospitals, but systematically trusting those hospitals concentrates demand and increases their service burden.
The shortcut amplification further highlights the need for tiered referral systems: allowing patients to bypass intermediate hospitals and go directly to the largest regional centers can overwhelm scarce specialist capacity, whereas staged referrals distribute demand across levels of care.

\subsection{Real medical data}
\label{sec:isic2019}

\begin{wrapfigure}{r}{0.42\textwidth}
\vspace{-2.7\baselineskip}
\centering
\captionsetup{hypcap=false,justification=raggedright,singlelinecheck=false}
\captionof{table}{\small\textbf{ISIC2019 balanced accuracy} (\%; three-seed mean).
Parentheses show the \texttt{EdgeDRM} change from \texttt{DFedAvgM}.}
\vskip -0.1in
\label{tab:isic2019-current}
\small
\setlength{\tabcolsep}{4pt}
\begingroup
\renewcommand{\arraystretch}{0.9}
\begin{tabular}{lcc}
\toprule
Graph & \texttt{DFedAvgM} & \texttt{EdgeDRM} \\
\midrule
\textsc{Path} & 28.45 & \textbf{28.62} \textcolor{green!50!black}{(+0.18)} \\
\textsc{Cycle} & 27.20 & \textbf{28.35} \textcolor{green!50!black}{(+1.15)} \\
\textsc{RR3} & 27.45 & \textbf{27.76} \textcolor{green!50!black}{(+0.31)} \\
\textsc{Complete} & \textbf{28.93} & 28.79 \textcolor{red!75!black}{(-0.14)} \\
\bottomrule
\end{tabular}
\endgroup
\vspace{-0.8\baselineskip}
\end{wrapfigure}
To assess \texttt{EdgeDRM} beyond synthetic benchmarks, we conduct a real-world skin-lesion classification study on ISIC2019 using the six source-center assignments distributed with FLamby~\citep{NEURIPS2022_232eee8e}.
Unlike the CIFAR benchmarks, where client heterogeneity is imposed through synthetic label-skew partitions, these clients follow source and acquisition metadata from real medical image collections.
Every test image is evaluated at all six entry clients, and we report pooled balanced accuracy, which weights recall equally across classes to account for the pronounced class imbalance in ISIC2019.
Appendix~\ref{app:isic2019} provides the data provenance, split audit, graph, training, and evaluation details.
\texttt{EdgeDRM} matches or improves upon the graph-matched \texttt{DFedAvgM} baseline on three of the four topologies. 
On \textsc{Complete}, it is $0.14$ points lower, consistent with previous results, where full connectivity likewise narrows or eliminates the advantage of learned local referral.

\section{Conclusion and Limitations}

\texttt{EdgeDRM} formulates inference-time collaboration in decentralized federated learning as budgeted local expert referral.
A query moves only between neighbors on the communication graph, and only visited experts contribute to its final prediction.
Neighbor-local density-ratio comparisons make referral trainable without a central router or global expert scoring, while our analysis relates referral gains to comparison accuracy and graph structure.

Across six CIFAR settings, \texttt{EdgeDRM} ranks first among graph-matched decentralized methods in 16 of the 18 sparse-graph comparisons, with a mean accuracy difference of $+3.61$ percentage points relative to the strongest competing method in each comparison.
The \textsc{Complete} graph gives \texttt{EdgeDRM} its highest absolute accuracy in every setting, but its mean advantage over \texttt{DFedAvgM} disappears ($-0.49$ points).
Together, these results establish graph-constrained referral as a distinct inference setting and a useful mechanism under sparse connectivity, while showing that greater connectivity can improve absolute accuracy without increasing the relative benefit of learned referral.
Closing the gap between the theory and learned neural representations remains future work.

\bibliography{references}
\bibliographystyle{unsrtnat}

\appendix
\section{Related Work}
\label{app:related_work}

\paragraph{Federated routing and mixtures of experts.}
\texttt{FedDRM} is the closest starting point for our work.
It learns a globally indexed source-classification head and asks a server to compare all clients for every query \citep{wang2026feddrm}.
\texttt{FedSPM} instead uses a semiparametric mixture to represent inter- and intra-client heterogeneity \citep{wang2026fedspm}.
Federated mixture-of-experts methods also construct specialized predictors: \texttt{FedMix} adapts a client-specific subset of experts, whereas \texttt{DDOME} jointly trains gates and experts for on-the-fly selection \citep{reisser2021fedmix,farhat2023learning}.
These methods establish the value of specialization, but they do not address our information constraint: at inference, no server or client can evaluate a global gate.
\texttt{EdgeDRM} replaces global selection with a sequence of graph-supported comparisons and forms its prediction only from the experts reached by that sequence.

\paragraph{Personalized and decentralized federated learning.}
Personalized FL decides which information clients should share during training. \texttt{pFedGraph} learns a weighted collaboration graph at a server \citep{ye2023pfedgraph}.
Fully decentralized methods learn personalized models and sparse collaboration graphs, exchange sparse model masks, optimize budgeted collaborator selection, or form soft clusters through peer-to-peer updates \citep{zantedeschi2020fully,dai2022dispfl,kharrat2025dpfl,lin2025fedspd}.
\texttt{DFedAvgM} combines local momentum with neighbor model averaging \citep{sun2023dfedavgm}, while decentralized SGD theory characterizes gossip and local-update convergence \citep{koloskova2020unified}.
These approaches determine how models collaborate during training and ordinarily return a model at each client.
Our problem begins after training: for each new query, the system must decide whether and where that query should move.
We therefore use \texttt{DFedAvgM} as the primary graph-matched training baseline and isolate referral with \texttt{EdgeDRM-NR}, evaluated on the same \texttt{EdgeDRM} predictors.

\paragraph{Expert referral and local graph search.}
Expert-referral networks are closer to our deployment protocol because they model accept-or-forward decisions on sparse graphs.
\texttt{DIEL} learns topic-dependent neighbor expertise from observed referral outcomes through a distributed active-learning rule \citep{khudabukhsh2016distributed}.
Decentralized-search theory studies when local connections yield short query paths \citep{kleinberg2000smallworld,ma2019performance}, and learned similarity-graph routing can use graph-wide search supervision \citep{baranchuk2019learning}.
\texttt{EdgeDRM} requires neither logged referral outcomes nor target paths.
Its local scores are learned from the static datasets held at neighboring clients, and its movement budget makes the access constraint explicit.

\paragraph{Density ratios and pairwise consistency.}
Classifier-based density-ratio estimation recovers pointwise density ratios from prior-corrected posterior odds \citep{sugiyama2012density}.
Statistical Hodge theory shows how pairwise comparison flows decompose into gradient, cyclic, and harmonic components \citep{jiang2011statistical}.
We do not propose a new density-ratio estimator in isolation.
Instead, one potential per client induces compatible edge differences, while balanced endpoint classification estimates pairwise density-ratio terms and the sample-count prior is added only at inference.
We then analyze how estimation error, representation disagreement, and graph structure affect the resulting referral policy.

\section{Derivation of the \texttt{EdgeDRM} Objective}
\label{app:edgedrm-derivation}

This appendix establishes the precise connection between empirical likelihood (EL) and the global client-source conditional likelihood in Section~\ref{sec:edgedrm-model}.
The equality requires profiling both the baseline masses and the density-ratio intercepts at fixed representations and slopes.
We then derive the graph-local conditional objective used by \texttt{EdgeDRM}.

\subsection{Empirical likelihood under the reference density-ratio model}

Our objective is to estimate the label model and the representation density ratios by maximizing their joint empirical likelihood while treating the reference distribution as an unrestricted nuisance parameter.
Recall that client $i\in V=[m]$ holds $D_i=\{(x_{ir},y_{ir})\}_{r=1}^{n_i}$ with $n_i>0$, total sample size $N=\sum_{i=1}^m n_i$, and sample proportion $\rho_i=n_i/N$.
Given the shared encoder $g_\theta$, write $z_{ir}=g_\theta(x_{ir})$ for the observed representation and let $P_i^Z$ denote the induced representation distribution on client $i$.
For client $i$, the reference density-ratio model specifies
\begin{equation*}
dP_i^Z(z)=\varphi_i(z)dF(z),
\qquad
\varphi_i(z)=\exp\{b_i+\vartheta_i^\top z\},
\qquad
F=P_1^Z,\quad b_1=0,\quad\vartheta_1=0.
\end{equation*}
The reference constraint makes $\varphi_1\equiv1$, whereas the remaining density ratios describe how each client distribution tilts the common baseline $F$.
For the empirical-likelihood construction, treat the common representations as fixed.
Let $\xi_1,\ldots,\xi_{n_z}$ be the distinct pooled representations, let $k_t=\#\{(u,r):z_{ur}=\xi_t\}$, and write $t(u,r)$ for the index satisfying $z_{ur}=\xi_{t(u,r)}$.
Thus, $k_t\geq1$ and $\sum_t k_t=N$.
We represent the unknown reference distribution by assigning mass $p_t=F(\{\xi_t\})$ to each distinct pooled representation $\xi_t$.
Let $P_u^{Y,Z}$ denote the joint distribution of the label and representation at client $u$.
Following the empirical-likelihood construction in \citet{wang2026feddrm}, we first define the joint empirical log-likelihood of the observed labels and representations as
\begin{equation*}
\ell_{\rm EL}(p,\theta,\omega,b,\vartheta)
:=\sum_{u=1}^m\sum_{r=1}^{n_u}
\log P_u^{Y,Z}(\{(y_{ur},z_{ur})\}).
\end{equation*}
Because the conditional label model depends on $X$ only through $Z=g_\theta(X)$, the joint distribution factorizes into the conditional label distribution and the representation marginal:
\begin{equation*}
P_u^{Y,Z}(\{(y_{ur},z_{ur})\})
=P_u^{Y\mid X}(\{y_{ur}\}\mid x_{ur};\theta,\omega_u)
P_u^Z(\{z_{ur}\}).
\end{equation*}
The density-ratio model and the definition of the empirical baseline masses then give
\begin{equation*}
P_u^Z(\{z_{ur}\})
=\varphi_u(z_{ur})F(\{z_{ur}\})
=\varphi_u(z_{ur})p_{t(u,r)}.
\end{equation*}
Substituting these two identities into the joint empirical log-likelihood yields
\begin{equation*}
\begin{aligned}
\ell_{\rm EL}(p,\theta,\omega,b,\vartheta)
&=\sum_{u=1}^m\sum_{r=1}^{n_u}
\log\!\left[
P_u^{Y\mid X}(\{y_{ur}\}\mid x_{ur};\theta,\omega_u)
P_u^Z(\{z_{ur}\})
\right]\\
&=\sum_{u=1}^m\sum_{r=1}^{n_u}
\log\!\left[
P_u^{Y\mid X}(\{y_{ur}\}\mid x_{ur};\theta,\omega_u)
\varphi_u(z_{ur})p_{t(u,r)}
\right]\\
&=\sum_{u=1}^m\sum_{r=1}^{n_u}
\log P_u^{Y\mid X}(\{y_{ur}\}\mid x_{ur};\theta,\omega_u)
+\sum_{u=1}^m\sum_{r=1}^{n_u}\log\varphi_u(z_{ur})
+\sum_{u=1}^m\sum_{r=1}^{n_u}\log p_{t(u,r)}\\
&=\sum_{u=1}^m\sum_{r=1}^{n_u}
\log P_u^{Y\mid X}(\{y_{ur}\}\mid x_{ur};\theta,\omega_u)
+\sum_{u=1}^m\sum_{r=1}^{n_u}\log\varphi_u(z_{ur})
+\sum_{t=1}^{n_z} k_t\log p_t.
\end{aligned}
\end{equation*}
The empirical-likelihood problem maximizes this joint objective over the model parameters and baseline masses subject to the requirement that $F$ and every tilted measure $P_i^Z$ are probability distributions.
The baseline normalization and the client-specific normalizations therefore impose
\begin{equation*}
p_t\geq0,\qquad
\sum_{t=1}^{n_z} p_t=1,\qquad
\sum_{t=1}^{n_z} p_t\varphi_i(\xi_t)=1,\quad i=2,\ldots,m.
\end{equation*}
Assigning one mass to each distinct representation avoids counting the same atom of $F$ more than once, while the multiplicity $k_t$ preserves its full contribution to the likelihood.
A finite objective value further requires $p_t>0$ for every observed atom.
Because the normalization constraints couple $p$ and $b$, profiling only $p$ at fixed $(b,\vartheta)$ does not generally recover the client-source likelihood pointwise in $b$.
The exact equivalence therefore requires profiling both the baseline masses and the density-ratio intercepts, as established next.

\subsection{Exact equivalence after profiling masses and intercepts}

At fixed representations and slopes, our objective is to profile the baseline masses and density-ratio intercepts and express the resulting empirical likelihood as a conditional source-classification likelihood.
To define this conditional likelihood, introduce the source variable $C$ and let $P_\rho^{C,Z}$ denote the pooled source--representation distribution with
\begin{equation*}
P_\rho^C(\{i\})=\rho_i,
\qquad
P_\rho^{Z\mid C}(\cdot\mid i)=P_i^Z,
\qquad
\rho_i=n_i/N.
\end{equation*}
Under a normalized density-ratio model, Bayes' rule gives the posterior source probability
\begin{equation*}
s_i(z)
:=P_\rho^{C\mid Z}(\{i\}\mid z)
=\frac{\rho_i\varphi_i(z)}{\sum_{c=1}^m\rho_c\varphi_c(z)}.
\end{equation*}
For any finite candidate coefficients, this softmax expression remains a valid conditional source model even when no normalized baseline $F$ exists for those coefficients.
For the profiling identity, fix $(\theta,\omega)$ and the resulting pooled representations, and fix finite slopes $\vartheta$ with $\vartheta_1=0$.
Define the joint empirical-likelihood feasible set
\begin{equation*}
\mathcal A_\vartheta
 :=\left\{(p,b)\in(0,\infty)^{n_z}\times\mathbb R^m:
b_1=0,\quad
\sum_{t=1}^{n_z}p_t=1,\quad
\sum_{t=1}^{n_z}p_t\exp\{b_i+\vartheta_i^\top\xi_t\}=1\ \text{for }i=2,\ldots,m
\right\}.
\end{equation*}
Our claim is that jointly maximizing the empirical likelihood over the baseline masses and intercepts is equivalent to maximizing the conditional source likelihood over the intercepts alone:
\begin{equation}
\max_{(p,b)\in\mathcal A_\vartheta}
\ell_{\rm EL}(p,\theta,\omega,b,\vartheta)
=\sum_{u=1}^m\sum_{r=1}^{n_u}
\log P_u^{Y\mid X}(\{y_{ur}\}\mid x_{ur};\theta,\omega_u)
+\max_{\substack{b\in\mathbb R^m\\b_1=0}}
\sum_{u=1}^m\sum_{r=1}^{n_u}\log s_u(z_{ur})
+C_{\rm emp},
\label{eq:fully-profiled-el}
\end{equation}
where $C_{\rm emp}:=\sum_{t=1}^{n_z} k_t\log(k_t/N)-\sum_{i=1}^m n_i\log\rho_i$.
Thus, after profiling both the unrestricted baseline masses and the density-ratio intercepts, the two optimized objectives differ only by the fixed conditional-label term and the data-dependent constant $C_{\rm emp}$.
We follow the profiling argument of \citet{wang2026feddrm}: first profile the baseline masses, then identify their Lagrange multipliers through the intercept score equations, and finally verify existence and global optimality.

\paragraph{Profiling the baseline masses.}
Fix an intercept vector $b$ for which the empirical-likelihood constraints admit positive masses.
The conditional-label term and $\sum_{u,r}\log\varphi_u(z_{ur})$ do not depend on $p$, so profiling $p$ only requires maximizing $\sum_t k_t\log p_t$.
Using the equivalent constraints $\sum_t p_t=1$ and $\sum_t p_t\{\varphi_i(\xi_t)-1\}=0$ for $i=2,\ldots,m$, introduce the multiplier $\mu$ for the baseline normalization and $\nu_i$ for the normalization of client $i$.
The Lagrangian for the $p$-dependent term is
\begin{equation*}
\sum_{t=1}^{n_z}k_t\log p_t
-N\mu\left(\sum_{t=1}^{n_z}p_t-1\right)
-N\sum_{i=2}^m\nu_i\sum_{t=1}^{n_z}p_t\{\varphi_i(\xi_t)-1\}.
\end{equation*}
Differentiating with respect to $p_t$ gives
\begin{equation*}
0=\frac{k_t}{p_t}-N\mu-N\sum_{i=2}^m\nu_i\{\varphi_i(\xi_t)-1\}.
\end{equation*}
Multiplying this equation by $p_t$, summing over $t$, and applying the constraints gives $0=N-N\mu$, so $\mu=1$.
Hence the profiled masses have the form
\begin{equation}
p_t
=\frac{k_t}{N\left[1+\sum_{i=2}^m\nu_i\{\varphi_i(\xi_t)-1\}\right]},
\qquad t=1,\ldots,n_z,
\label{eq:profiled-baseline-masses}
\end{equation}
Because $\sum_t k_t\log p_t$ is strictly concave in the positive masses and the constraints are affine in $p$, any feasible solution of these equations is the unique profiled mass vector.
For fixed $b$, the multipliers $\nu_i$ must satisfy the normalization constraints and generally depend on $b$.

\paragraph{Identifying the multipliers at the joint optimum.}
Suppose a joint interior optimum exists, denote it by $(p^\star,b^\star)$, and evaluate $\varphi_i$ at $b^\star$ throughout this paragraph.
Differentiation of the full Lagrangian with respect to $b_i$ gives, for $i=2,\ldots,m$,
\begin{equation*}
0
=n_i-N\nu_i\sum_{t=1}^{n_z}p_t^\star\varphi_i(\xi_t)
=n_i-N\nu_i.
\end{equation*}
The second equality uses the normalization $\sum_t p_t^\star\varphi_i(\xi_t)=1$.
Therefore, the multiplier at the joint optimum is $\nu_i=n_i/N=\rho_i$.
Substituting these values into~\eqref{eq:profiled-baseline-masses} and using $\varphi_1\equiv1$ gives
\begin{equation}
p_t^\star
=\frac{k_t}{N\left[1+\sum_{i=2}^m\rho_i\{\varphi_i(\xi_t)-1\}\right]}
=\frac{k_t}{N\sum_{c=1}^m\rho_c\varphi_c(\xi_t)}.
\label{eq:optimal-baseline-masses}
\end{equation}
Substitution into the empirical log-likelihood then yields
\begin{equation*}
\begin{aligned}
\ell_{\rm EL}(p^\star,\theta,\omega,b^\star,\vartheta)
&=\sum_{u=1}^m\sum_{r=1}^{n_u}
\log P_u^{Y\mid X}(\{y_{ur}\}\mid x_{ur};\theta,\omega_u)
+\sum_{u=1}^m\sum_{r=1}^{n_u}
\log\frac{\varphi_u(z_{ur})}{\sum_{c=1}^m\rho_c\varphi_c(z_{ur})}
+\sum_{t=1}^{n_z}k_t\log(k_t/N)\\
&=\sum_{u=1}^m\sum_{r=1}^{n_u}
\log P_u^{Y\mid X}(\{y_{ur}\}\mid x_{ur};\theta,\omega_u)
+\sum_{u=1}^m\sum_{r=1}^{n_u}\log s_u(z_{ur})
+\sum_{t=1}^{n_z}k_t\log(k_t/N)
-\sum_{i=1}^m n_i\log\rho_i\\
&=\sum_{u=1}^m\sum_{r=1}^{n_u}
\log P_u^{Y\mid X}(\{y_{ur}\}\mid x_{ur};\theta,\omega_u)
+\sum_{u=1}^m\sum_{r=1}^{n_u}\log s_u(z_{ur})
+C_{\rm emp}.
\end{aligned}
\end{equation*}
This calculation identifies the objective at a joint optimum, but it does not by itself establish that such an optimum exists or is globally optimal.

\paragraph{Existence and uniqueness of the intercept maximizer.}
We now establish that the conditional source log-likelihood has a unique finite maximizer and that the masses in~\eqref{eq:optimal-baseline-masses} are feasible at this intercept vector.
Write $h_{it}=\log\rho_i+\vartheta_i^\top\xi_t$ and $M=\max_{i,t}|h_{it}|<\infty$.
For $b_1=0$, let $B=\max_i b_i$, $A=\min_i b_i$, and $n_{\min}=\min_i n_i>0$.
To establish existence, we bound each numerator score above by $b_u+M$ and each log normalizer below by $B-M$:
\begin{equation*}
\sum_{u=1}^m\sum_{r=1}^{n_u}\log s_u(z_{ur})
\leq 2NM+\sum_i n_i b_i-NB
\leq 2NM-n_{\min}(B-A)
\leq 2NM-n_{\min}\|(b_2,\ldots,b_m)\|_\infty.
\end{equation*}
The bound shows that the conditional source log-likelihood tends to $-\infty$ as the free intercepts diverge, so continuity guarantees a finite maximizer.
To establish uniqueness, take any nonzero perturbation $v$ of the free intercepts and extend it by setting $v_1=0$.
The corresponding Hessian quadratic form is
\begin{equation*}
v^\top\nabla_b^2\!\left\{\sum_{u=1}^m\sum_{r=1}^{n_u}\log s_u(z_{ur})\right\}v
=-\sum_{t=1}^{n_z} k_t\operatorname{Var}_{C\sim s(\xi_t)}(v_C)<0.
\end{equation*}
Because $s_i(\xi_t)>0$ for every client and $v$ is nonconstant, $\operatorname{Var}_{C\sim s(\xi_t)}(v_C)>0$ for every $t$.
The Hessian is therefore negative definite, which proves strict concavity and uniqueness of the intercept maximizer.
Let $b^\star(\vartheta)$ denote this unique maximizer, consistently with the notation used above, and evaluate $\varphi_i$ at $b^\star$.
Its score equations verify that the masses in~\eqref{eq:optimal-baseline-masses} satisfy every empirical-likelihood constraint.
For each $i\geq2$, the score equation is
\begin{equation*}
0=\partial_{b_i}\!\left\{\sum_{u=1}^m\sum_{r=1}^{n_u}\log s_u(z_{ur})\right\}
=n_i-\rho_i\sum_{t=1}^{n_z} k_t\frac{\varphi_i(\xi_t)}{\sum_{c=1}^m\rho_c\varphi_c(\xi_t)}.
\end{equation*}
Because $\rho_i=n_i/N$, these equations imply $\sum_t p_t^\star\varphi_i(\xi_t)=1$ for every $i\geq2$.
Their baseline normalization follows from $\varphi_1\equiv1$ and $\rho_1>0$:
\begin{equation*}
1=\sum_t p_t^\star\sum_{c=1}^m\rho_c\varphi_c(\xi_t)
=\rho_1\sum_t p_t^\star+\sum_{i=2}^m\rho_i,
\qquad\text{so}\qquad\sum_t p_t^\star=1.
\end{equation*}
Thus $(p^\star,b^\star)$ is a feasible pair.

\paragraph{Global-optimality verification.}
It remains to verify that this feasible pair is globally optimal rather than merely stationary.
For any $(p,b)\in\mathcal A_\vartheta$, the normalization constraints give $\sum_t p_t\sum_{c=1}^m\rho_c\varphi_c(\xi_t)=1$.
Jensen's inequality therefore gives
\begin{equation*}
\begin{aligned}
&\ell_{\rm EL}(p,\theta,\omega,b,\vartheta)
-\sum_{u=1}^m\sum_{r=1}^{n_u}\log P_u^{Y\mid X}(\{y_{ur}\}\mid x_{ur};\theta,\omega_u)
-\sum_{u=1}^m\sum_{r=1}^{n_u}\log s_u(z_{ur})-C_{\rm emp}\\
&\quad=\sum_{t=1}^{n_z} k_t\log\frac{Np_t\sum_{c=1}^m\rho_c\varphi_c(\xi_t)}{k_t}\\
&\quad\leq N\log\left\{\sum_{t=1}^{n_z}\frac{k_t}{N}
\frac{Np_t\sum_{c=1}^m\rho_c\varphi_c(\xi_t)}{k_t}\right\}=0.
\end{aligned}
\end{equation*}
Equality holds precisely for the masses in~\eqref{eq:optimal-baseline-masses}.
The conditional source log-likelihood is uniquely maximized at $b^\star$, so this upper bound is attained by $(p^\star,b^\star)$ and proves~\eqref{eq:fully-profiled-el}.

\paragraph{Scope of the equivalence.}
Equation~\eqref{eq:fully-profiled-el} is an optimization identity obtained by profiling both the baseline masses and the intercepts at fixed representations and slopes.
At fixed representations, optimizing the slopes as well preserves equality of the suprema, and finite global maximizers correspond through the displayed masses whenever such maximizers exist.
However, existence of a finite slope maximizer is a separate condition and can fail under source separation.
Moreover, the identity does not assert that profiling only the baseline masses at a fixed $(b,\vartheta)$ yields the conditional-label and conditional-source terms in~\eqref{eq:fully-profiled-el} plus the same constant $C_{\rm emp}$.
For example, take two clients with one observation each, $z_{11}=0$, $z_{21}=1$, and $(\varphi_2(0),\varphi_2(1))=(1/2,3/2)$, while $\varphi_1\equiv1$.
The EL constraints force both baseline masses to be $1/2$, giving marginal likelihood $3/8$, whereas the source likelihood is $(2/3)(3/5)=2/5$ and $C_{\rm emp}=0$.
The model in this example is already reference normalized, so fixing the reference does not remove the need to optimize the intercepts.
If all pooled representations are distinct, $C_{\rm emp}=-N\log N-\sum_i n_i\log\rho_i$.
With repeated representations, $C_{\rm emp}$ is constant for a fixed representation but can change if learning $\theta$ changes their multiplicities.
Consequently, we use~\eqref{eq:loss} as a conditional-label plus source-classification objective, while interpreting the EL equivalence only at fixed representations rather than as an unrestricted identity throughout encoder training.

\subsection{From the global source loss to graph-local comparisons}

The global source posterior $s_i(z)$ cannot be evaluated from graph-local information because its denominator requires scores from all clients.
Our objective is therefore to construct an edgewise source model that uses only the two endpoint scores while preserving their pairwise density ratio.
Fix an edge $e=\{i,j\}$ and consider a balanced endpoint experiment that first selects $i$ or $j$ with probability $1/2$ and then draws a representation from the selected endpoint.
Let $q_{ij}^{\rm bal}(z)$ denote the conditional probability that representation $z$ originated from endpoint $j$ in this experiment.
Under the density-ratio model, Bayes' rule gives
\begin{equation*}
q_{ij}^{\rm bal}(z)
:=\varphi_j(z)/\{\varphi_i(z)+\varphi_j(z)\},
\end{equation*}
whose logit is the pairwise log-density ratio
\begin{equation*}
\operatorname{logit}q_{ij}^{\rm bal}(z)
=\log\{\varphi_j(z)/\varphi_i(z)\}
=(b_j-b_i)+(\vartheta_j-\vartheta_i)^\top z.
\end{equation*}
For the orientation $i<j$ and endpoint $u\in\{i,j\}$, define the endpoint-normalized source loss
\begin{equation}
\mathcal L_e^{(u)}
=-\frac1{n_u}\sum_{r=1}^{n_u}
\left[
\mathbf1\{u=i\}\log\{1-q_{ij}^{\rm bal}(z_{ur})\}
+\mathbf1\{u=j\}\log q_{ij}^{\rm bal}(z_{ur})
\right].
\label{eq:local-edge-loss}
\end{equation}
Because the endpoint prior is balanced, the negative empirical conditional log-likelihood averages the two endpoint-specific losses equally:
\begin{equation*}
\begin{aligned}
\tfrac12\mathcal L_e^{(i)}+\tfrac12\mathcal L_e^{(j)}
&=-\frac{1}{2n_i}\sum_{r=1}^{n_i}\log\{1-q_{ij}^{\rm bal}(z_{ir})\}\\
&\quad-\frac{1}{2n_j}\sum_{r=1}^{n_j}\log q_{ij}^{\rm bal}(z_{jr}).
\end{aligned}
\end{equation*}
This equal weighting removes the sample-proportion offset, so the population logit targets only the pairwise density ratio.
Moreover, each edgewise likelihood requires no graph-wide normalization because it depends only on the endpoint scores.
Summing the edgewise terms therefore produces an implementable conditional composite likelihood rather than the global source likelihood in~\eqref{eq:fully-profiled-el}.

\subsection{Balanced graph-local objective used in training}
\label{app:weighted-graph-objective}

Our objective is now to combine label prediction with the edgewise source likelihood in a client-local form whose average assigns the intended weight to every task and edge term.
The main text presents these components separately to distinguish prediction from source classification, whereas the implementation combines them using the weighting derived below.
Client $u$ maintains a local representation $g_{\theta_u}$ and evaluates $z_{ur}:=g_{\theta_u}(x_{ur})$ on its own examples.
The local parameters need not coincide after a finite number of neighbor-mixing rounds; setting all $\theta_u=\theta$ recovers the common-representation model used in the derivation above.

We now combine the endpoint source loss in~\eqref{eq:local-edge-loss} with the conditional-label probabilities in~\eqref{eq:local-predictive-model}.

We next choose the scaling that makes the client-local losses aggregate to the desired graph-level objective.
Let $\lambda\in(0,1]$ denote the weight assigned to task prediction, leaving weight $1-\lambda$ for source classification.
To weight the $m$ task losses uniformly and the $|E|$ edge losses uniformly after averaging over clients, set
\begin{equation}
\kappa:=\frac{m(1-\lambda)}{2|E|}.
\label{eq:balanced-edge-scale}
\end{equation}
Let $a_i=(\vartheta_i,b_i)$ denote the routing potential owned by client $i$.
The resulting client-local contribution is
\begin{equation}
\mathcal L_i
=-\frac{\lambda}{n_i}\sum_{r=1}^{n_i}
\log P_i^{Y\mid X}(\{y_{ir}\}\mid x_{ir};\theta_i,\omega_i)
+\kappa\sum_{j\in\mathcal N(i)}\mathcal L_{\{i,j\}}^{(i)},
\label{eq:edgedrm-local-objective}
\end{equation}
and averaging these contributions over clients yields
\begin{equation}
\mathcal L_G
=\frac{1}{m}\sum_i\mathcal L_i
=-\frac{\lambda}{m}\sum_i\frac1{n_i}\sum_{r=1}^{n_i}
\log P_i^{Y\mid X}(\{y_{ir}\}\mid x_{ir};\theta_i,\omega_i)
+(1-\lambda)/(2|E|)
\sum_{e=\{i,j\}\in E}
\{\mathcal L_e^{(i)}+\mathcal L_e^{(j)}\}.
\label{eq:edgedrm-graph-objective}
\end{equation}
Thus, each endpoint loss is a mean over its local data, the two endpoints of every edge receive equal weight, and all edges receive equal weight without an additional degree normalization.
The reported experiments use $\lambda=0.9$, while $\lambda=1$ gives the task-only ablation.
For fixed features, the prediction-head and routing-potential subproblems separate.

To optimize~\eqref{eq:edgedrm-graph-objective} using only neighbor communication, each endpoint differentiates every incident edge loss with respect to both endpoint potentials and returns the neighbor-owned component.
The resulting unregularized routing-potential gradient received by owner $i$ is
\begin{equation*}
G_i=\nabla_{a_i}\mathcal L_i+
\kappa\sum_{j\in\mathcal N(i)}
\nabla_{a_i}\mathcal L_{\{i,j\}}^{(j)}.
\end{equation*}
This expression equals $\nabla_{a_i}\sum_u\mathcal L_u$ in unrestricted potential coordinates when all endpoint contributions are evaluated at the same local-step snapshot.
For the reference-constrained theoretical update, $a_1$ remains fixed at zero and only the components $G_i$ for $i\geq2$ are used.
The private head receives $\nabla_{\omega_i}\mathcal L_i$, while the frozen encoder receives no gradient.

The balanced endpoint likelihood estimates the pairwise density ratio without incorporating the sample-count prior.
Because the global source posterior includes that prior, inference recovers the appropriate source odds by adding it to the learned comparison:
\begin{equation*}
\log\{s_j(z)/s_i(z)\}
=\log(\rho_j/\rho_i)+\log\{\varphi_j(z)/\varphi_i(z)\}
=\log(n_j/n_i)+(b_j-b_i)+(\vartheta_j-\vartheta_i)^\top z.
\end{equation*}
Therefore, training excludes the sample-count prior and inference adds $\log(n_j/n_i)$ exactly once.

Finally, when the frozen local representations differ, the fitted reverse-edge scores need not be antisymmetric.
Section~\ref{sec:theory} separates this representation-disagreement error from coefficient-estimation error.

\section{Proofs and Additional Theory}
\label{app:proofs}

Appendix~\ref{app:parameter-proof} proves the coefficient-estimation guarantee, Appendix~\ref{app:source-odds-proof} derives the comparison identities, and Appendix~\ref{app:structural-proof} proves the referral guarantee.
For the routing arguments, fix the positive training counts and query $(x,s)$.
We use the route $\mathcal R_L(x,s)=(\widehat s_0=s,\ldots,\widehat s_T)$ from \eqref{eq:budgeted-route}, with terminal client $\widehat s_T$, and the scores $u_t$ and weights $w_t$ from Section~\ref{sec:referral-history}.
Empty sums are zero and $\max\varnothing=-\infty$.

\subsection{Proof of Theorem~\ref{thm:parameter-estimation}}
\label{app:parameter-proof}

\textbf{Statistical conditions.}
Fix a connected graph with $m\ge2$, the number of classes, the representation and coefficient dimensions, and the task weight $\lambda\in(0,1)$.
The common representation is fixed independently of the fitting samples.
Within each client the observations are i.i.d. from a fixed distribution $P_i^{X,Y}$, and different clients are independent.
The deterministic allocations satisfy $\rho_{i,N}:=n_i/N\to\rho_i^*>0$; for a fixed $N$, these are exactly the proportions $\rho_i$ used by the inference-time prior in Section~\ref{sec:edgedrm-optimization}.
The conditional-label model and the common-reference DRM are correctly specified by the same finite coefficient vector $\zeta^\star$ for all $N$.
For identifiability, we use $F=P_1^Z$ and fix $b_1=0$ and $\vartheta_1=0$, so that $\varphi_1\equiv1$.
In identifiable coordinates the limiting population objective defined below has a unique interior minimizer $\zeta^\star$.
On a compact convex parameter region containing this point in its interior, the observation losses and their first two derivatives have integrable envelopes; the score at $\zeta^\star$ has finite second moment, and the population Hessian is positive definite.
Losses are measurable and twice continuously differentiable on a neighborhood of the region, and empirical minimizers are measurable.
These compact-region regularity conditions apply directly to the balanced training losses below.
The optimization conditions below specify existence of a finite empirical minimizer and containment of the iterates.

\textbf{Observation contributions.}
We decompose the graph objective $\mathcal L_G$ in \eqref{eq:edgedrm-graph-objective} into observation-level contributions.
For a generic observation $O=(X,Y)$, let $z=g_\theta(X)$.
For candidate coefficients $\zeta$, write
\[
q_{ij}^{\rm bal}(z;\zeta)
=\left[1+\exp\{-(b_j-b_i)-(\vartheta_j-\vartheta_i)^\top z\}\right]^{-1}.
\]
At $\zeta=\zeta^\star$, the parameterized quantity $q_{ij}^{\rm bal}(z;\zeta)$ reduces to the main-text posterior $q_{ij}^{\rm bal}(z)$.
The density-tilt coefficients retain their definition in \eqref{eq:common-reference-drm}; no sample-count offset appears in this training logit.
For an edge $e=\{v,w\}$, orient its endpoints so that $v<w$.
The observation-level contribution needed for the asymptotic analysis is
\[
\begin{aligned}
\ell_{G,i}(O;\zeta)
&=-\frac{\lambda}{m}\log P_i^{Y\mid X}(\{Y\}\mid X;\theta,\omega_i)\\
&\quad-\frac{1-\lambda}{2|E|}
\sum_{\substack{e=\{v,w\}\ni i\\v<w}}
\left[
\mathbf1\{i=v\}\log\{1-q_{vw}^{\rm bal}(g_\theta(X);\zeta)\}
+\mathbf1\{i=w\}\log q_{vw}^{\rm bal}(g_\theta(X);\zeta)
\right],\\
\mathcal L_G(\zeta)
&=\sum_i\frac1{n_i}\sum_{r=1}^{n_i}\ell_{G,i}(O_{ir};\zeta),
\qquad
R_G(\zeta)=\sum_i\mathbb E_i\ell_{G,i}(O_{i1};\zeta),
\end{aligned}
\]
Here $O_{ir}=(X_{ir},Y_{ir})$, and $\mathbb E_i$ denotes expectation under $P_i^{X,Y}$.
All terms evaluated on the same observation remain inside the same $\ell_{G,i}$.
Thus $\mathbb E\mathcal L_G=R_G$ for every allocation: the sample proportions affect the estimator's variance and the inference-time prior, but not the balanced population training objective.

\textbf{Identifiable coordinates and the gradient update.}
The routing coordinates are the free blocks $a_i=(\vartheta_i,b_i)$ for $i=2,\ldots,m$, with $a_1=0$ held fixed.
For each private head, retain $\sum_k\alpha_{ik}=0$ and $\sum_k\beta_{ik}=0$ and use orthonormal coordinates on this class-contrast subspace.
All parameter derivatives and norms in Theorem~\ref{thm:parameter-estimation} are taken in these coordinates.
All unqualified vector norms below denote the Euclidean norm.
Choosing client 1 as the density reference preserves every pairwise ratio: from any common baseline $F_0$ with coefficients $(b_i^0,\vartheta_i^0)$, take $F=P_1^Z$ and $(b_i,\vartheta_i)=(b_i^0-b_1^0,\vartheta_i^0-\vartheta_1^0)$.
Then $dP_i^Z/dF=\exp\{b_i+\vartheta_i^\top z\}$ remains a normalized density.
The private-head gradient sums to zero across classes, so common-step unpenalized updates preserve its class-contrast subspace.
The analyzed routing update acts only on the free blocks $a_2,\ldots,a_m$ and keeps $a_1=0$.

We analyze synchronous, unpenalized plain SGD in these free coordinates for Stage II.
Write $\mathcal L_i$ and $\mathcal L_e^{(u)}$ for the method's client-local and endpoint-normalized losses.
For each free routing-potential owner $i\ne1$, the gradient is exactly
\[
G_i^{a}
=\nabla_{a_i}\mathcal L_i
 +\kappa\sum_{j\in\mathcal N(i)}
 \nabla_{a_i}\mathcal L_{\{i,j\}}^{(j)},
\qquad \kappa=\frac{m(1-\lambda)}{2|E|}.
\]
The head owner receives $G_i^{\omega}=\nabla_{\omega_i}\mathcal L_i$, expressed in the head coordinates above.
Client 1 still contributes its task gradient and the source gradients of its observations to neighboring free routing blocks.
With the same snapshot for every endpoint contribution, concatenating the free routing blocks and all head-coordinate blocks gives $\nabla\sum_i\mathcal L_i=m\nabla\mathcal L_G$, including the gradients returned by neighbors.
Divide the assembled gradient by $m$; equivalently, a step of size $\eta_r$ on the normalized gradient uses step size $\eta_r/m$ on the raw owner gradients.
The analyzed update is
\begin{equation}
\zeta^{(r+1)}=\zeta^{(r)}-\eta_r\widehat G_r,
\qquad
\mathbb E[\widehat G_r\mid\mathcal F_r]
=\nabla\mathcal L_G(\zeta^{(r)}),
\label{eq:stage-two-sgd}
\end{equation}
where the data and initialization are included in the optimization history $\mathcal F_r$.
At each update, independently across clients, sample $B_i^{\rm mb}$ local indices uniformly with replacement and average $\nabla\ell_{G,i}$ over them, then sum these averages over $i$.
This gives exactly the conditional mean in \eqref{eq:stage-two-sgd}.
If $\widehat P_i$ is the empirical distribution at client $i$, the conditional covariance of this normalized mini-batch gradient is
\[
\sum_i\frac1{B_i^{\rm mb}}
 \operatorname{Var}_{\widehat P_i}
       \{\nabla\ell_{G,i}(O;\zeta^{(r)})\}.
\]
Its trace is the conditional noise second moment.
This covariance includes the correlations between incident-edge losses evaluated on the same observation.
A variance bound for raw owner gradients must be divided by $m^2$ to give the corresponding normalized bound.
Other samplers are covered only when the stated conditional unbiasedness and joint noise bound hold.

\textbf{Optimization conditions.}
The optimization analysis uses the following finite-minimizer, curvature, and bounded-iterate conditions.
On a data-and-initialization event whose probability tends to one, let $\widehat\zeta_N$ be a finite empirical minimizer with zero coordinate gradient, and let a convex region $\mathcal K_N$ contain this minimizer and the initial iterate, with
\[
\mu I\preceq\nabla^2\mathcal L_G(\zeta)\preceq\mu^{-1}I,
\qquad 0<\mu\le1,
\qquad \zeta\in\mathcal K_N.
\]
Any uniform positive lower and finite upper curvature bounds admit this form after choosing $\mu$ sufficiently small.
Assume the conditional noise second moment is at most $\sigma^2$ while the iterate is in this region.
The constants are uniform in $N$, and $\|\zeta^{(0)}-\widehat\zeta_N\|=O_p(1)$.
As in the bounded-iterate premise, all iterates up to the analyzed budget must remain in the region, either almost surely conditional on these data or with probability tending to one.
For high-probability containment, the proof uses the first exit time from $\mathcal K_N$.
On a bounded-logit region, positive-definite augmented-feature Gram matrices for every task head and balanced edge mixture, together with graph connectivity, are sufficient for positive curvature in the free reference-client and private-head coordinates.
Existence of a finite minimizer and containment of the iterates are separate assumptions; complete separation and unidentified design directions are excluded.
The curvature condition is local to $\mathcal K_N$.

\begin{proof}[Proof of Theorem~\ref{thm:parameter-estimation}]
\textbf{Step 1: A common population minimizer.}
Under the correctly specified DRM, the main-text endpoint-source posterior is
\[
q_{ij}^{\rm bal}(z)
=\frac{\varphi_j(z)}{\varphi_i(z)+\varphi_j(z)}.
\]
Let $\chi_e(z)$ be the gradient in identifiable coordinates of $(b_j-b_i)+(\vartheta_j-\vartheta_i)^\top z$.
It does not depend on the parameter value, since these logits are affine.
At parameter $\zeta$, the source-label gradient for an observation with representation $z$ from endpoint $u$ is $\{q_{ij}^{\rm bal}(z;\zeta)-\mathbf1\{u=j\}\}\chi_e(z)$.
Consequently, at $\zeta^\star$,
\[
\begin{aligned}
&\tfrac12\mathbb E_i[q_{ij}^{\rm bal}(Z)\chi_e(Z)]
 +\tfrac12\mathbb E_j[(q_{ij}^{\rm bal}(Z)-1)\chi_e(Z)]\\
&\quad=\int\bigl\{
   \tfrac12\{\varphi_i(z)+\varphi_j(z)\}q_{ij}^{\rm bal}(z)
       -\tfrac12\varphi_j(z)\bigr\}\chi_e(z)\,dF(z)=0.
\end{aligned}
\]
The two equally weighted endpoint means therefore cancel exactly.
For the task term, write $p_i(z;\zeta)$ for the model class-probability vector.
Let $A_i(z)$ be the Jacobian of the $K$ task logits with respect to $\zeta$, and let $\mathbf e_Y$ denote the class indicator vector.
Correct specification gives
\[
\mathbb E_i\!\left[
\nabla_\zeta\{-\log P_i^{Y\mid X}(\{Y\}\mid X;\theta,\omega_i)\}
\mid Z\right]
=A_i(Z)^\top\{p_i(Z;\zeta^\star)-\mathbb E_i[\mathbf e_Y\mid Z]\}=0.
\]
Thus, with $\mathcal S_i(O)=\nabla\ell_{G,i}(O;\zeta^\star)$ and $\overline{\mathcal S}_i=\mathbb E_i\mathcal S_i(O)$, the population-score identity is
\[
\nabla R_G(\zeta^\star)=\sum_i \overline{\mathcal S}_i=0,
\qquad\text{while generally }\overline{\mathcal S}_i\ne0.
\]
Convexity already makes $\zeta^\star$ a minimizer.
More explicitly, $R_G(\zeta)-R_G(\zeta^\star)$ is the sum of
\[
\frac{\lambda}{m}\sum_i\mathbb E_i
 \operatorname{KL}\{p_i(Z;\zeta^\star)\,\|\,p_i(Z;\zeta)\}
\]
and
\[
\frac{1-\lambda}{|E|}\sum_{e=\{i,j\}}
\int\tfrac12\{\varphi_i(z)+\varphi_j(z)\}
 \operatorname{KL}\!\left\{
 \operatorname{Ber}(q_{ij}^{\rm bal}(z))\,\|\,
 \operatorname{Ber}(q_{ij}^{\rm bal}(z;\zeta))\right\}\,dF(z).
\]
All weights are positive.
Equality requires the correct task probabilities and density-score differences on every edge.
The stated identification condition therefore makes $\zeta^\star$ the unique minimizer of the balanced population objective $R_G$.

\textbf{Step 2: Score covariance.}
The endpoint score second moments follow from the information identity.
At the true balanced posterior,
\[
\begin{aligned}
&\tfrac12\sum_{u\in e}
  \mathbb E_u\!\left[
  (q_{ij}^{\rm bal}(Z)-\mathbf1\{u=j\})^2\|\chi_e(Z)\|^2\right]\\
&\quad=\int\tfrac12\{\varphi_i(z)+\varphi_j(z)\}
    q_{ij}^{\rm bal}(z)\{1-q_{ij}^{\rm bal}(z)\}
                  \|\chi_e(z)\|^2\,dF(z)<\infty.
\end{aligned}
\]
The right-hand side is the trace of the edge's population Hessian before its positive graph-level coefficient.
All Hessian contributions are positive semidefinite, so finiteness follows from the stated Hessian envelope.
Each endpoint term is finite.
The task-score second moment follows from the stated moment condition (and from its correctly specified softmax information identity).
Uniform integrable gradient envelopes also follow by expanding each gradient from $\zeta^\star$ over the compact region using its Hessian envelope.

Here are the exact score and Hessian for the balanced objective.
All probabilities in the following displays are evaluated at $\zeta^\star$ and the displayed representation:
\[
\begin{aligned}
\mathcal S_i(O)
&=\frac{\lambda}{m}A_i(Z)^\top\{p_i(Z;\zeta^\star)-\mathbf e_Y\}\\
&\quad+\frac{1-\lambda}{2|E|}\sum_{e=\{v,w\}\ni i,\ v<w}
 (q_{vw}^{\rm bal}(Z)-\mathbf1\{i=w\})\chi_e(Z),\\
H_G:=\nabla^2R_G(\zeta^\star)
&=\frac{\lambda}{m}\sum_i\mathbb E_i\!
 \left[A_i(Z)^\top
 \{\operatorname{diag}(p_i(Z;\zeta^\star))-p_i(Z;\zeta^\star)p_i(Z;\zeta^\star)^\top\}
 A_i(Z)\right]\\
&\quad+\frac{1-\lambda}{|E|}\sum_{e=\{i,j\}}
 \int\tfrac12\{\varphi_i(z)+\varphi_j(z)\}
 q_{ij}^{\rm bal}(z)\{1-q_{ij}^{\rm bal}(z)\}
                 \chi_e(z)\chi_e(z)^\top\,dF(z).
\end{aligned}
\]
Set
\[
\begin{aligned}
\Omega_i&=\mathbb E_i\!
 [(\mathcal S_i(O)-\overline{\mathcal S}_i)(\mathcal S_i(O)-\overline{\mathcal S}_i)^\top],\\
J_{G,N}&=\operatorname{Var}\!
       \{\sqrt N\nabla\mathcal L_G(\zeta^\star)\}
       =\sum_i\frac{N}{n_i}\Omega_i,\\
J_G&=\lim_N J_{G,N}
       =\sum_i(\rho_i^*)^{-1}\Omega_i,\\
\Sigma_G&=H_G^{-1}J_GH_G^{-\top}.
\end{aligned}
\]
Here $\Omega_i$ is centered within client stratum $i$, whose score mean $\overline{\mathcal S}_i$ may be nonzero.
If $\mathcal S_e^{(i)}=\{q_{vw}^{\rm bal}(Z)-\mathbf1\{i=w\}\}\chi_e(Z)$ for $e=\{v,w\}$, $v<w$, then the routing part of $\Omega_i$ contains the full double sum
\[
\left(\frac{1-\lambda}{2|E|}\right)^2
 \sum_{e,f\ni i}
 \operatorname{Cov}_i(\mathcal S_e^{(i)},\mathcal S_f^{(i)}).
\]
The terms with $e\ne f$ account for the reuse of each observation on incident edges.
The full vector definition of $\Omega_i$ also retains task contributions and their covariance with routing contributions.
Under correct task specification, the task--routing cross-covariances vanish because the task score has conditional mean zero given $Z$.
In general $J_G$ is not the same as $H_G$.

\textbf{Step 3: Statistical consistency and asymptotic normality.}
The uniform-law conditions give, on the compact parameter region,
\[
\sup_\zeta|\mathcal L_G(\zeta)-R_G(\zeta)|\to_p0,
\]
and the analogous uniform convergence for the Hessians.
Uniqueness and interior positive curvature yield $\widehat\zeta_N\to_p\zeta^\star$.
On any sufficiently small sphere about $\zeta^\star$, the population objective exceeds its value at $\zeta^\star$ uniformly, and so does the empirical objective with probability tending to one.
Convexity then excludes a better empirical minimizer outside that sphere, by considering the segment to $\zeta^\star$.

The exact cancellation of total score means gives
\[
\mathbb G_N:=\sqrt N\nabla\mathcal L_G(\zeta^\star)
=\sum_i\frac{\sqrt N}{n_i}\sum_{r=1}^{n_i}
            \{\mathcal S_i(O_{ir})-\overline{\mathcal S}_i\}.
\]
The centered summands have independent client strata.
For each client the ordinary finite-variance multivariate i.i.d. central limit theorem applies to $\mathcal S_i-\overline{\mathcal S}_i$.
Since $\sqrt{N/n_i}\to(\rho_i^*)^{-1/2}$ and the number of clients is fixed, $\mathbb G_N\xrightarrow{d}\mathcal N(0,J_G)$.
Thus the allocation proportions enter the asymptotic covariance, although they do not enter the balanced population objective.

On the event that the empirical minimizer is interior, its exact first-order condition gives
\[
0=\nabla\mathcal L_G(\zeta^\star)
       +\widehat H_N(\widehat\zeta_N-\zeta^\star),\qquad
\widehat H_N=\int_0^1\nabla^2\mathcal L_G
 \{\zeta^\star+t(\widehat\zeta_N-\zeta^\star)\}\,dt\to_p H_G.
\]
Consequently,
\begin{equation}
\sqrt N(\widehat\zeta_N-\zeta^\star)
=-H_G^{-1}\sqrt N\nabla\mathcal L_G(\zeta^\star)+o_p(1)
\xrightarrow{d}\mathcal N(0,\Sigma_G).
\label{eq:composite-estimator-clt}
\end{equation}
\textbf{Step 4: Optimization error.}
Take $\eta_r=2/\{\mu(r+r_0)\}$ with $r_0\ge2/\mu^2$, so that $r_0\ge2$ and $\eta_r\le\mu$.
An integral Hessian expansion along the segment between $\zeta$ and $\widehat\zeta_N$ gives, for $0<\eta^{\rm opt}\le\mu$,
\[
\|\zeta-\eta^{\rm opt}\nabla\mathcal L_G(\zeta)-\widehat\zeta_N\|
\le(1-\mu\eta^{\rm opt})\|\zeta-\widehat\zeta_N\|.
\]
Conditional unbiasedness makes the cross term with the gradient noise zero.
Let $\tau_N$ be the first iterate index outside $\mathcal K_N$ and put
\[
v_r=\mathbb E_{\rm opt}\!
 \left[\|\zeta^{(r)}-\widehat\zeta_N\|^2
                      \mathbf1\{\tau_N>r\}\right].
\]
This is the usual mean-square error when containment holds almost surely.
The event $\{\tau_N>r\}$ is measurable with respect to $\mathcal F_r$.
Since $\{\tau_N>r+1\}\subseteq\{\tau_N>r\}$,
\[
v_{r+1}\le\mathbb E_{\rm opt}\!\left[
\mathbf1\{\tau_N>r\}\,
\mathbb E\!\left[
\|\zeta^{(r+1)}-\widehat\zeta_N\|^2\mid\mathcal F_r
\right]\right].
\]
Applying conditional unbiasedness on the $\mathcal F_r$-measurable event inside this expectation gives
\[
v_{r+1}\le\left(1-\frac{2}{r+r_0}\right)^2v_r
             +\frac{4\sigma^2}{\mu^2(r+r_0)^2}
\le\left(1-\frac{2}{r+r_0}\right)v_r
             +\frac{4\sigma^2}{\mu^2(r+r_0)^2}.
\]
The last inequality uses $r+r_0\ge2$.
Induction gives
\begin{equation}
v_R\le\frac{C_{\rm opt}}{R+r_0},\qquad
C_{\rm opt}=\max\left\{r_0v_0,\frac{4\sigma^2}{\mu^2}\right\}.
\label{eq:stage-two-convergence}
\end{equation}
Indeed, writing $d=r+r_0$, if $v_r\le C/d$ with $C\ge4\sigma^2/\mu^2$, then $v_{r+1}\le C/d-(2C-4\sigma^2/\mu^2)/d^2\le C/d-C/d^2\le C/(d+1)$.
The initial condition starts the induction.
Since $C_{\rm opt}=O_p(1)$, conditional Markov's inequality and the stated containment condition $\Pr(\tau_N\le R)\to0$ imply
\[
\zeta^{(R)}-\widehat\zeta_N=O_p(R^{-1/2}).
\]
With high-probability containment, \eqref{eq:stage-two-convergence} bounds the stopped error $v_R$; the exit probability is controlled separately by the containment assumption.

\textbf{Step 5: The SGD output.}
Combining this expansion with \eqref{eq:stage-two-convergence} yields
\[
\|\zeta^{(R)}-\zeta^\star\|
\le\|\widehat\zeta_N-\zeta^\star\|
  +\|\zeta^{(R)}-\widehat\zeta_N\|
=O_p(N^{-1/2}+R^{-1/2}),
\]
which proves \eqref{eq:coefficient-total-error}.
If $R/N\to\infty$, the stopped mean-square bound and the containment condition imply $\sqrt N(\zeta^{(R)}-\widehat\zeta_N)=o_p(1)$, so \eqref{eq:composite-estimator-clt} proves \eqref{eq:algorithm-clt}.
This proves Theorem~\ref{thm:parameter-estimation}.

\end{proof}

\textbf{Full-gradient updates.}
For full gradients and $0<\eta^{\rm opt}\le\mu$, the same curvature and iterate-containment conditions give
\begin{equation}
\|\zeta^{(R)}-\widehat\zeta_N\|
\le(1-\mu\eta^{\rm opt})^R\|\zeta^{(0)}-\widehat\zeta_N\|.
\label{eq:stage-two-full-gradient}
\end{equation}
Adding the statistical error gives $O_p\{N^{-1/2}+(1-\mu\eta^{\rm opt})^R\}$.
Thus $\sqrt N(1-\mu\eta^{\rm opt})^R\to0$ suffices for the same normal limit.
For fixed $c>0$, it is sufficient to take $R\ge(1/2+c)(\mu\eta^{\rm opt})^{-1}\log N$; when $\mu\eta^{\rm opt}=1$ the contraction is zero after one update.

\subsection{Balanced source fitting and prior-weighted comparison identities}
\label{app:source-odds-proof}

For $e=\{i,j\}$, $i<j$, the source-label distribution in the edge risk assigns probability $1/2$ to each endpoint.
Bayes' rule under this balanced mixture gives
\[
q_{ij}^{\rm bal}(z)
=\frac{\varphi_j(z)}{\varphi_i(z)+\varphi_j(z)},\qquad
\operatorname{logit}q_{ij}^{\rm bal}(z)
=\log\frac{\varphi_j(z)}{\varphi_i(z)}.
\]
The equally weighted binary cross-entropy is minimized by this posterior, which contains no sample-count term and agrees with~\eqref{eq:logitqij}.
Correct specification, connectivity, and the stated identifiable-design condition identify the routing coefficients relative to the fixed client 1.
Correctly specified task cross-entropies identify the private-head contrasts.
At inference, adding $\log(\rho_j/\rho_i)$ to the fitted balanced logit converts the learned density ratio into the desired prior-weighted source odds.
Thus the sample-count prior is absent from training and appears exactly once in deployment.

With a common fitted representation $z=g_{\widehat\theta}(x)$,
\[
u_t(x,s)=\log\frac{\rho_{\widehat s_t}}{\rho_s}
 +(\widehat b_{\widehat s_t}-\widehat b_s)
 +(\widehat\vartheta_{\widehat s_t}-\widehat\vartheta_s)^\top z.
\]
The identity holds for any fitted coefficients with a common representation.
If the DRM holds for this representation and its density-score differences are recovered exactly, then
\[
e^{u_t(x,s)}=\frac{\rho_{\widehat s_t}\varphi_{\widehat s_t}(z)}
 {\rho_s\varphi_s(z)},\qquad
w_t(x,s)=\frac{\rho_{\widehat s_t}\varphi_{\widehat s_t}(z)}
 {\sum_{\ell=0}^T\rho_{\widehat s_\ell}\varphi_{\widehat s_\ell}(z)}.
\]
The history weights normalize the prior-weighted densities over the visited clients.
With different local encoders, the count terms still telescope but the remaining score sum can be path dependent.

For a reference representation $g_\theta(x)$ and its coefficients, adding and subtracting the fitted comparison at this representation gives
\begin{equation}
\begin{aligned}
\widehat\Delta_{ij}(x)-\Delta_{ij}(x)
&=e_{ij}^{\rm coef}(x)
 +(\widehat\vartheta_j-\widehat\vartheta_i)^\top
   \{g_{\widehat\theta_i}(x)-g_\theta(x)\},\\
e_{ij}^{\rm coef}(x)
&=(\widehat b_j-b_j)-(\widehat b_i-b_i)\\
&\quad+\{(\widehat\vartheta_j-\vartheta_j)
 -(\widehat\vartheta_i-\vartheta_i)\}^\top g_\theta(x).
\end{aligned}
\label{eq:score-error-decomposition}
\end{equation}
Both scores use the same finite-$N$ count prior, so their known offsets cancel exactly, irrespective of how quickly the proportions converge.
If the coefficient term is bounded by $\epsilon_a$, the fitted slope contrast by $M_\vartheta$, and the representation difference by $\delta_g$, the score error is at most $\epsilon_a+M_\vartheta\delta_g$.
For common representations and the same count prior, each comparison is affine in the identifiable coefficients.
Using client 1 as the reference for both fitted and true coefficients, its error is bounded by $2\|(g_\theta(x)^\top,1)^\top\|\max_i\|\widehat a_i-a_i\|$.
Theorem~\ref{thm:parameter-estimation} therefore proves \eqref{eq:comparison-error-bound} at a fixed query; bounded representations give a uniform version.
For different local encoders,~\eqref{eq:score-error-decomposition} remains an algebraic decomposition.
However, these encoders also change the representation distributions used for coefficient fitting.
The coefficient term may therefore contain effects of this change in addition to statistical estimation error, and Theorem~\ref{thm:parameter-estimation} does not automatically bound it.
Extending the rate to this setting requires controlling both the coefficient-fitting error and the representation-disagreement term.

The reverse-edge discrepancy is
\begin{equation}
\widehat\Delta_{ij}(x)+\widehat\Delta_{ji}(x)
=(\widehat\vartheta_j-\widehat\vartheta_i)^\top
 \{g_{\widehat\theta_i}(x)-g_{\widehat\theta_j}(x)\}.
\label{eq:reciprocity-defect}
\end{equation}
Here the two count offsets and the intercept differences cancel even with different encoders.
Let $\Delta_{ij}^{\rm pop}(x)$ denote the comparison obtained by replacing the training proportions $\rho_i$ with positive reference-population proportions $\rho_i^{\rm pop}$.
Its additional prior error is $\log(\rho_j/\rho_j^{\rm pop})-\log(\rho_i/\rho_i^{\rm pop})$, so
\begin{equation}
|\widehat\Delta_{ij}(x)-\Delta_{ij}^{\rm pop}(x)|
\le\epsilon_a+M_\vartheta\delta_g+
2\max_i|\log(\rho_i/\rho_i^{\rm pop})|.
\label{eq:prior-score-error}
\end{equation}
The last term measures the difference between the training-source prior and the reference population prior.

\subsection{Proof of Theorem~\ref{thm:structural-gains} and structural examples}
\label{app:structural-proof}

\textbf{From estimation error to comparison tolerance.}
Under Theorem~\ref{thm:parameter-estimation}, \eqref{eq:comparison-error-bound} bounds the maximum comparison error over all graph edges at the fixed query $x$.
By the definition of $O_p$, for every fixed $\delta\in(0,1)$ there is a finite constant $C_\delta>0$ such that, for all sufficiently large $N,R$, the event
\begin{equation}
\mathcal E_\delta:=\left\{
\max_{\{i,j\}\in E}|\widehat\Delta_{ij}(x)-\Delta_{ij}(x)|
\le C_\delta\left(N^{-1/2}+R^{-1/2}\right)=\epsilon
\right\}
\label{eq:realized-edge-error}
\end{equation}
has probability at least $1-\delta$ over the training data and SGD randomness.
The constant may depend on the fixed query, graph, and model, but not on $N$, $R$, or the realized route.
With the common representation, reverse comparisons are negatives of each other, so this event controls both edge directions.
It therefore controls every candidate edge on the adaptively selected route without an additional probability bound for individual steps.
All remaining arguments are deterministic on $\mathcal E_\delta$; the reference potentials and $H_\epsilon$ are fixed given the training counts and query.
For an independently random frozen representation, the argument is conditional on that representation, as in Theorem~\ref{thm:parameter-estimation}.

\textbf{Definition of the depth.}
Suppress $x$ in potentials and let $H_\epsilon$ be the largest integer $q\ge0$ for which there exists a path $(i_0=s,i_1,\ldots,i_q)$ satisfying, at every step $r<q$,
\[
i_{r+1}\in\mathcal N(i_r),\qquad
h_{i_{r+1}}>h_{i_r},\qquad
h_{i_{r+1}}\ge\max_{j\in\mathcal N(i_r)}h_j-2\epsilon.
\]
The one-node path is permitted.
Strict increase excludes revisits, so $0\le H_\epsilon\le m-1$.
The maximum is attained because the graph is finite.
Throughout the proof, $H_\epsilon$ denotes this exact maximum for the given graph, potentials, entry, and tolerance.

\begin{proof}[Proof of Theorem~\ref{thm:structural-gains}]
Work on $\mathcal E_\delta$ and put $d_\epsilon=\gamma-2\epsilon>0$.

\textbf{Step 1: Accepted moves and stopping.}
Every accepted edge satisfies
$\Delta_{ij}\ge\widehat\Delta_{ij}-\epsilon>\eta-\epsilon\ge0$.
Hence every previously visited client has smaller potential than the current client, and every positive-gain neighbor is unvisited.
At a nonoptimal client $i$, let $j^\star$ maximize the true neighboring gain.
By \eqref{eq:minimum-ascent}, it is eligible and
$\widehat\Delta_{ij^\star}\ge\gamma-\epsilon>\eta$.
Thus the route cannot stop at $i$ while budget remains.
For the selected neighbor $j$,
\[
\Delta_{ij}\ge\widehat\Delta_{ij}-\epsilon
\ge\widehat\Delta_{ij^\star}-\epsilon
\ge\Delta_{ij^\star}-2\epsilon\ge d_\epsilon.
\]
This also gives $h_j\ge\max_{k\in\mathcal N(i)}h_k-2\epsilon$, so the realized route is admissible in the definition of $H_\epsilon$.
Consequently $T\le\min\{L,H_\epsilon\}$ and
\begin{equation}
\Gamma_L(x,s)\ge T(\gamma-2\epsilon).
\label{eq:uniform-ascent-gain}
\end{equation}
At a global maximizer, every candidate fitted score is at most $\epsilon\le\eta$, so no further move is accepted.
An early stop therefore occurs at a maximizer.

\textbf{Step 2: The budget bound.}
If $T=L$, \eqref{eq:uniform-ascent-gain} gives $\Gamma_L\ge Ld_\epsilon$.
If $T<L$, Step 1 gives $\Gamma_L=h^\star-h_s$.
Thus
\begin{equation}
\Gamma_L\ge\min\{h^\star-h_s,L(\gamma-2\epsilon)\}.
\label{eq:budgeted-ascent-gain}
\end{equation}

\textbf{Step 3: The total potential gap and ascent depth.}
Let $(i_0=s,\ldots,i_{H_\epsilon})$ attain the maximum in the definition of $H_\epsilon$.
Its terminal vertex is a maximizer: otherwise, a highest-potential neighbor has positive gain by \eqref{eq:minimum-ascent}, is unvisited by strict ascent, and extends the admissible path.
For each $r<H_\epsilon$, the near-best-neighbor condition gives
\[
h_{i_{r+1}}-h_{i_r}
\ge\max_{j\in\mathcal N(i_r)}(h_j-h_{i_r})-2\epsilon
\ge d_\epsilon.
\]
Telescoping along this path yields the separate total-gap bound
\[
h^\star-h_s
=\sum_{r=0}^{H_\epsilon-1}(h_{i_{r+1}}-h_{i_r})
\ge H_\epsilon d_\epsilon.
\]
Together with Step 2, this proves \eqref{eq:unified-referral-gain}:
\[
\Gamma_L\ge\min\{h^\star-h_s,Ld_\epsilon\}
\ge d_\epsilon\min\{L,H_\epsilon\},
\]

\textbf{Step 4: A sufficient budget for reaching a maximizer.}
Suppose $L\ge H_\epsilon$ and the terminal client is suboptimal.
Step 1 then implies budget exhaustion, so $T=L$.
Together with $T\le H_\epsilon$, this gives $T=L=H_\epsilon$.
A highest-potential neighbor would extend the realized admissible path, contradicting the definition of $H_\epsilon$.
Thus the terminal client is a maximizer.
The same arguments include $L=0$ and an optimal entry, and hold under any tie breaking among score maximizers.
\end{proof}

\textbf{Examples and the proportional-ascent alternative.}
The following sufficient conditions concern the prior-adjusted potentials at a fixed query.
Their conclusions hold on the comparison-error event used above.

\begin{assumption}[Uniform ascent]
\label{ass:uniform-ascent}
Every nonoptimal client has a neighbor with gain at least $\gamma>0$, as in \eqref{eq:minimum-ascent}.
\end{assumption}

\begin{assumption}[Referral hierarchy]
\label{ass:referral-hierarchy}
A spanning tree is oriented toward a maximizing root.
Each nonroot client $i$ has a parent $p_x(i)$ such that
\[
\Delta_{i,p_x(i)}\ge\gamma,\qquad
\Delta_{i,p_x(i)}-
 \max_{j\in\mathcal N(i)\setminus\{p_x(i)\}}\Delta_{ij}
\ge\kappa_{\rm hier}>0.
\]
\end{assumption}
When $\kappa_{\rm hier}>2\epsilon$, only the parent can belong to an admissible ascent path.
Thus $H_\epsilon$ is exactly the number of parent edges from $s$ to the root.
Under the theorem's threshold condition,
\begin{equation}
T=\min\{L,H_\epsilon\},\qquad \Gamma_L\ge\gamma\min\{L,H_\epsilon\}.
\label{eq:hierarchical-gain}
\end{equation}
The preferred-parent condition therefore identifies ascent depth with the distance to the root in this hierarchy.

\begin{assumption}[Mixture of expert types]
\label{ass:hospital-mixture}
Clients of the same type $\tau(i)$ have equal potential.
There is a best type $\tau^\star(x)$ with $h_j-h_i\ge\gamma>0$ whenever $\tau(j)=\tau^\star(x)$ and $\tau(i)\ne\tau^\star(x)$.
Every nonbest-type client has a best-type neighbor.
\end{assumption}
If $\gamma>\eta+\epsilon$ and $\eta\ge\epsilon$, then $\gamma>2\epsilon$.
Only best-type neighbors are admissible, so $H_\epsilon=1$ for a nonoptimal entry and $H_\epsilon=0$ for an optimal entry.
Therefore
\begin{equation}
\Gamma_L=(h^\star-h_s)\mathbf1\{L\ge1\}.
\label{eq:mixture-density-gain}
\end{equation}
Equal representation distributions and equal source priors within each type suffice for the equal-potential condition.

\begin{assumption}[Graph--PL ascent]
\label{ass:graph-pl}
There is $\kappa_{\rm PL}\in(0,1]$ such that each nonoptimal client satisfies $\max_{j\in\mathcal N(i)}\Delta_{ij}\ge\kappa_{\rm PL}(h^\star-h_i)$.
\end{assumption}
This is an alternative to uniform ascent: the available improvement scales with the remaining potential gap.
Under this assumption, the candidate error bound and $\eta\ge\epsilon$ alone imply
\begin{equation}
\Gamma_L\ge\{1-(1-\kappa_{\rm PL})^L\}
 \left[h^\star-h_s-\frac{\eta+\epsilon}{\kappa_{\rm PL}}\right]_+.
\label{eq:graph-pl-gain}
\end{equation}
To prove it, hold a stopped route at its terminal node until index $L$ only for the analysis, and put $d_t=h^\star-h_{\widehat s_t}$.
On a moving step, the preceding noisy-greedy inequality gives $d_{t+1}\le(1-\kappa_{\rm PL})d_t+2\epsilon$.
On a stopped step, the best neighboring gain is at most $\eta+\epsilon$, so $\kappa_{\rm PL}d_t\le\eta+\epsilon$ and $d_{t+1}=d_t\le(1-\kappa_{\rm PL})d_t+\eta+\epsilon$.
Using $2\epsilon\le\eta+\epsilon$ and solving the common recursion proves \eqref{eq:graph-pl-gain}.
In particular an early stop satisfies
\begin{equation}
T<L\quad\Longrightarrow\quad
h^\star-h_{\widehat s_T}\le\frac{\eta+\epsilon}{\kappa_{\rm PL}}.
\label{eq:graph-pl-stopping}
\end{equation}
When $T=L$, the sharper $2\epsilon$ numerator follows from the moving-step recursion.
Use $(1-\kappa_{\rm PL})^0=1$, including $\kappa_{\rm PL}=1$.
The holding steps are an analytical convention; $T$, the prediction history, and communication counts are those of the original route.

\section{Experiment Details}
\label{app:experiment-details}

The subsections below collect the data construction, communication graphs, optimization schedules, baseline provenance, sensitivity analyses, inference diagnostics, and real-data protocols used throughout Section~\ref{sec:experiments}.

\subsection{CIFAR Data, Graphs, Baselines, and Traffic Construction}
\label{app:frozen-data}

The benchmark uses the client construction from the released \texttt{FedDRM} code \citep{wang2026feddrm}.
CIFAR-10 has ten labels, CIFAR-20 maps the CIFAR-100 fine labels to 20 superclasses, and CIFAR-100 retains all 100 fine labels.
The official training and test collections are pooled before partitioning among eight clients using one of two label-skew constructions.
In the SPC construction, label-sorted data are divided into equal-sized, label-homogeneous shards, and each client receives five uniformly sampled shards for CIFAR-10 or 25 for CIFAR-20/100, yielding equal client dataset sizes.
In the Dirichlet construction, client-wise class proportions are randomly drawn from a Dirichlet distribution with concentration $0.3$, yielding heterogeneous label marginals and unequal client dataset sizes.
Each client's assigned images are shuffled and divided into 70\% training and 30\% evaluation data.
The latter is split into 20\% validation and 80\% test data.
Training, validation, and test indices are disjoint.

The eight deterministic client transformations combine a gamma-correction exponent in $\{0.6,1.4\}$, hue shift in $\{-0.1,0.1\}$, and saturation in $\{0.5,1.5\}$.
The CIFAR models use a GroupNorm ResNet-18 followed by a 256-dimensional task projection with LayerNorm and a local classifier.
Within each paired run, all clients share a random initialization; no pretrained \texttt{FedDRM} parameters are loaded.
Both CIFAR training phases retain the released random-resized crop to $32\times32$ pixels with scale $0.9$--$1.0$, followed by the fixed client color transformation and normalization.
Evaluation uses one seeded crop per image, shared across paired models and readouts.
Freezing fixes encoder parameters and buffers; it does not fix the augmented input at every training access.
The fixed evaluation view and non-source entry assignments are shared between each 800+200 model and its \texttt{DFedAvgM} control.
The source ID is used to construct the evaluation and for observer-only diagnostics; it is not given to the routing or history-weighting rule.

The \textsc{Path} has edges $\{\{i,i+1\}:1\leq i<8\}$, and the \textsc{Cycle} adds $\{8,1\}$.
The fixed \textsc{RR4} graph has neighbor sets
\begin{equation*}
\begin{split}
&\mathcal N(1)=\{2,3,4,5\},\quad
\mathcal N(2)=\{1,3,6,7\},\quad
\mathcal N(3)=\{1,2,4,8\},\quad
\mathcal N(4)=\{1,3,5,7\},\\
&\mathcal N(5)=\{1,4,6,8\},\quad
\mathcal N(6)=\{2,5,7,8\},\quad
\mathcal N(7)=\{2,4,6,8\},\quad
\mathcal N(8)=\{3,5,6,7\}.
\end{split}
\end{equation*}
The \textsc{Complete} graph contains all 28 undirected edges.
The sparse-graph movement budgets are the diameters 7, 4, and 2.
The primary \textsc{Complete} control uses $L=8$; $L=1$ is retained for the training-duration diagnostics.
Previously visited experts are excluded from subsequent moves.
No posttraining representation consensus is applied.

\paragraph{Thirty-two-client scaling.}
The CIFAR-10 scaling benchmark uses gamma-correction levels $\{0.6,1.0,1.4\}$, hue shifts $\{-0.15,0,0.15\}$, saturation levels $\{0.4,1.0,1.6\}$, and two-bit posterization for clients 28--32.
The pooled-data construction and within-client split remain the same.
The connected degree-four graph uses seed 82025, has 64 edges and diameter four, and retains the primary movement budget $L=8$.

\paragraph{Eight-client baseline measurements.}
For Table~\ref{tab:decentralized-comparison}, \texttt{Local} and \texttt{FedDRM} use 1000 rounds with ten local updates per round.
\texttt{DisPFL} uses 1000 rounds with five local updates per round and 50\% dynamic sparsity.
\texttt{FedSPD} uses 999 communication rounds with five local updates per round, followed by five final personalization updates, yielding 1000 round equivalents and 5000 local updates per client.
All four baselines are evaluated at fixed final checkpoints on the same seed-specific held-out test split used for \texttt{EdgeDRM} and \texttt{DFedAvgM}; Table~\ref{tab:decentralized-comparison} reports means and sample SDs over seeds 2025--2027.
\texttt{Local}, \texttt{DisPFL}, and \texttt{FedSPD} predict with the non-source entry model, whereas \texttt{FedDRM} uses its global client-selection head.
The round-equivalent budgets and evaluation splits are aligned, although the numbers of local updates differ.

\paragraph{Hospital-network construction.}
The 16/8/8 small/medium/large hospitals receive disjoint, exactly class-balanced subsets of the official CIFAR-10 training split with 250/1,000/4,000 examples each.
The layered graph has a cycle within each tier, one small--medium edge per small hospital, and a one-to-one medium--large matching (56 edges).
The shortcut graph adds one predetermined small--large edge per small hospital (72 edges).
Each graph uses its own shared-training prefix and frozen-encoder continuation.

The common diagnostic panel contains 20 official test images per class.
Each of the 200 distinct images is reused in 64 arrival records, yielding 12,800 records with exactly uniform entry counts across hospitals.
Seeds 2025--2027 vary training and entry assignments.
The prior-on and prior-off evaluations share the model, query panel, and arrival records; both exclude revisits and allow at most two referrals.
The panel is an existing diagnostic set rather than a newly blinded test set.
Workload quantities count the diagnostic calls made along these finite query paths.

\paragraph{Cumulative diagnostic workload.}
\label{app:hospital-workload}
Let $\mathcal V_q$ be the visited hospitals for arrival $q$, and let $\mathcal H_g$ contain the hospitals in tier $g$.
Each hospital supplies one prediction when visited, including the entry and terminal hospitals.
Because the route excludes revisits, its cumulative call counts are
\[
 C_i^{\rm call}=\sum_{q=1}^{Q}\mathbf 1\{i\in\mathcal V_q\},
 \qquad C_g^{\rm call}=\sum_{i\in\mathcal H_g}C_i^{\rm call},
 \qquad C_{\mathrm{all}}^{\rm call}=\sum_i C_i^{\rm call}
                       =\sum_{q=1}^{Q}(T_q+1),
\]
where $T_q$ is the number of moves and $Q=12{,}800$.
The reported tier share, peak individual share, and calls per arrival are, respectively,
\[
 S_g=\frac{C_g^{\rm call}}{C_{\mathrm{all}}^{\rm call}},
 \qquad P_{\max}=\frac{\max_i C_i^{\rm call}}{C_{\mathrm{all}}^{\rm call}},
 \qquad \overline C^{\rm call}=\frac{C_{\mathrm{all}}^{\rm call}}{Q}.
\]
Table~\ref{tab:hospital-traffic-details} reports both the large-tier share $S_{\mathrm{large}}$ and $100C_{\mathrm{large}}^{\rm call}/Q$, retaining the absolute large-hospital workload scale as the total number of calls changes.
It also reports $\overline C^{\rm call}$ and $P_{\max}$ to distinguish additional collaboration from concentration at any single hospital.
Each seed's counts and ratios are computed first, followed by their mean and sample SD across the three seeds.
Per-hospital counts reconstructed from the saved paths match the recorded classifier calls.
These quantities weight each classifier invocation equally.

\FloatBarrier
\begin{table}[H]
\centering
\caption{Accuracy and diagnostic workload in the hospital-network experiment.
Large share and peak share are percentages of all visited-classifier calls; large calls/100 is the absolute number of large-hospital calls per 100 arrivals.
Accuracy uses history-weighted predictions.
Entries are mean $\pm$ sample SD over seeds 2025--2027; \texttt{EdgeDRM-NR} workload quantities are fixed by uniform entry.}
\label{tab:hospital-traffic-details}
\scriptsize
\setlength{\tabcolsep}{3.5pt}
\begin{tabular*}{\textwidth}{@{\extracolsep{\fill}}lrrrrr}
\toprule
Condition & Accuracy (\%) & Calls/query & Large share (\%) & Large calls/100 & Peak share (\%) \\
\midrule
Layered: \texttt{EdgeDRM-NR} & $45.29\pm1.80$ & $1$ & $25$ & $25$ & $3.125$ \\
Layered: History, no prior & $45.74\pm1.60$ & $2.01\pm0.03$ & $19.55\pm0.44$ & $39.22\pm1.56$ & $5.04\pm0.44$ \\
Layered: History + prior & $47.36\pm3.06$ & $2.61\pm0.01$ & $52.02\pm0.12$ & $135.52\pm0.63$ & $7.14\pm0.20$ \\
\midrule
Shortcut: \texttt{EdgeDRM-NR} & $45.49\pm2.03$ & $1$ & $25$ & $25$ & $3.125$ \\
Shortcut: History, no prior & $46.12\pm1.73$ & $2.01\pm0.02$ & $20.36\pm0.33$ & $40.96\pm0.39$ & $4.61\pm0.59$ \\
Shortcut: History + prior & $47.93\pm2.53$ & $2.46\pm0.01$ & $69.49\pm0.10$ & $170.81\pm0.81$ & $9.77\pm0.18$ \\
\bottomrule
\end{tabular*}
\end{table}

\FloatBarrier

\subsection{CIFAR Training, Evaluation, and Additional Results}
\label{app:frozen-readouts}

The balanced-endpoint schedule below applies to the benchmark and hospital protocol.
The prediction-rule comparison uses CIFAR-100/Dirichlet-$0.3$, and the sensitivity sweep uses CIFAR-10/Dirichlet-$0.3$. Both use eight clients and seeds 2025--2027.
Each primary model uses its saved 200-round head and router checkpoint with the original 800-round encoder state.
The same continuation length is used for every graph and seed.
Evaluation preserves the training random-number and loader states.

\subsubsection{Training schedule and implementation details}
\label{app:frozen-training-protocol}
The routing loss assigns equal weight to the two endpoint means of every edge.
At inference, the deployed incident score adds the single sample-count offset $\log(n_j/n_i)$.
The pretraining duration and subsequent encoder freezing are experimental optimization choices, rather than defining components of \texttt{EdgeDRM}.

\paragraph{Stage I: 800 shared-training rounds.}
All clients start from the same model initialization.
Each round takes local SGD steps on the task loss and then mixes the full model $\psi_i=(\theta_i,\omega_i)$ by \eqref{eq:generic-decentralized-update}.
Both representation and classifier are exchanged, using only physical neighbors.
Thus \emph{shared} refers to full-model collaboration toward a common task objective, not a server-owned model or identical local parameters.
This stage uses the same graph and full-model mixing as the matched \texttt{DFedAvgM} control \citep{sun2023dfedavgm}.

\paragraph{Stage II: 200 rounds with frozen local encoders.}
Each expert inherits its own Stage-I model and sampler state, fixes $\bar\theta_i=\theta_i^{(800)}$, and initializes its routing potential $a_i=(\vartheta_i,b_i)$ to zero.
The encoder includes the backbone, feature projection, and feature normalization.
Its parameters and persistent buffers remain unchanged, and it runs in evaluation mode.
There are no encoder gradient updates or representation-mixing steps:
\begin{equation*}
\theta_i^{(800+r)}=\bar\theta_i,\qquad r=0,\ldots,200.
\end{equation*}
Each local step minimizes \eqref{eq:edgedrm-local-objective} only over the private classifier $\omega_i$ and the endpoint potentials.
Classifier heads are not exchanged, and potentials are not averaged.
Freezing keeps each local feature map stationary during this stage; it does not align different experts' feature maps.
The fitted representation used at inference is $\widehat\theta_i=\bar\theta_i$ under this experimental protocol.

The head and routing-potential updates instantiate Stage II of Section~\ref{sec:decentralized-training}, with no parameter mixing.
Routing gradients do not flow into the encoder or task head.
Training counts are exchanged before coefficient fitting.
After round 200, final potentials are published to neighbors.
All parameters remain fixed during inference, and no additional consensus phase is applied.

\begin{algorithm}[H]
\caption{Experimental 800+200 training schedule}
\label{alg:frozen-experimental-schedule}
\begin{algorithmic}[1]
\REQUIRE Graph $G$, neighbor weights $W$, local steps $B_{\rm loc}$, $\lambda$;
$R_{\rm sh}=800$, $R_{\rm per}=200$
\STATE Initialize all local task models identically
\FOR{$r=1,\ldots,R_{\rm sh}$}
  \STATE Each client takes $B_{\rm loc}$ task-only SGD steps
  \STATE Mix full models with physical neighbors using $W$
\ENDFOR
\STATE Inherit local models and sampler states; set $\bar\theta_i\leftarrow\theta_i^{(800)}$
\STATE Freeze encoder parameters and buffers in evaluation mode; initialize owner potentials to zero and exchange training counts along edges
\FOR{$r=1,\ldots,R_{\rm per}$}
  \FOR{$\ell=1,\ldots,B_{\rm loc}$}
    \STATE Owners send potential snapshots to incident neighbors
    \STATE Each endpoint computes local features with its frozen encoder
    \STATE Differentiate task loss for its private head and balanced endpoint losses for endpoint potentials
    \STATE Return neighbor-owned potential gradients along their edges
    \STATE Update only private heads and owner potentials with accumulated gradients
  \ENDFOR
\ENDFOR
\STATE Publish final potentials to neighbors; freeze heads and routers
\end{algorithmic}
\end{algorithm}

\paragraph{Fixed budget and learning-rate schedule.}
For the CIFAR protocol, each round has five local mini-batches of size 128.
SGD uses momentum $0.9$, reset each round, and weight decay $5\times10^{-4}$.
The phase-local learning rate is
\begin{equation}
\eta_r^{\rm exp}=0.005\{1+\cos(\pi(r-1)/H_{\rm lr})\},
\qquad
H_{\rm lr}=800\text{ in Stage I},\quad H_{\rm lr}=600\text{ in Stage II}.
\label{eq:two-stage-lr}
\end{equation}
Stage II restarts at $0.01$ but stops at round 200, with $\eta_{200}^{\rm exp}\simeq0.00752264$; it does not run to the end of the 600-round decay horizon.
The encoder remains frozen throughout this phase; the schedule applies only to heads and routing potentials.
The model at total round 1000 is the fixed endpoint, without test-based early stopping.
The finite neural optimizer with weight decay is distinct from the unpenalized fixed-representation estimator in Section~\ref{sec:theory}.

This schedule keeps local features stationary during coefficient fitting, but does not make different clients' representations identical or independent of the reused training data.
It therefore does not by itself satisfy the statistical assumptions of Theorem~\ref{thm:parameter-estimation}.

\subsubsection{Fixed-expert referral-score ablation}
The comparison in Figure~\ref{fig:referral-policy-ablation} reuses the final CIFAR-10 and CIFAR-100 Dirichlet-$0.3$ \textsc{RR4} checkpoints for all three seeds.
No model parameter is updated.
\texttt{EdgeDRM} routing uses the deployed incident score, including the single $\log(n_j/n_i)$ offset.
\texttt{EdgeDRM-NR} reports the same entry expert without referral.
Count-only replaces the learned part of every incident score by zero, leaving exactly $\log(n_j/n_i)$, and retains the learned rule's strict zero threshold.
Random routing samples uniformly from the unvisited incident neighbors and stops only at the two-move budget or when none remain; its seed is fixed before evaluation.
All three routing rules exclude revisits.
The primary prediction is the terminal expert's classifier, which changes only through the selected route.

\subsubsection{Sensitivity of $\lambda$}
There are three independent training seeds; graphs reuse each seed's partition.
Validation/test query counts are 3,600/14,403 for seed 2025, 3,603/14,400 for seed 2026, and 3,601/14,402 for seed 2027.
All prediction rules and controls use matched query indices, labels, source IDs, entry nodes, and shared-checkpoint hashes.
Accuracy is measured over the fixed non-source arrivals, with one view per query.

\FloatBarrier
\begin{table}[H]
\centering
\caption{Task--routing weight on CIFAR-10/Dirichlet-$0.3$, eight clients, \textsc{RR4}, and $L=2$.
All branches reuse the same seed-specific 800-round prefix and fit heads and routers for 200 rounds.
Accuracy is in percent, referral gain in percentage points, and calls per query.
Entries are final-checkpoint means $\pm$ sample SD across seeds 2025--2027.}
\label{tab:lambda-tradeoff}
\scriptsize
\setlength{\tabcolsep}{5.0pt}
\begin{tabular*}{\textwidth}{@{\extracolsep{\fill}}lrrrr}
\toprule
$\lambda$ & \texttt{EdgeDRM} & \texttt{EdgeDRM-NR} & Referral gain & Calls \\
\midrule
$0.10$ & $52.11\pm2.19$ & $35.01\pm3.55$ & $+17.11\pm3.98$ & $2.16\pm0.03$ \\
$0.25$ & $51.76\pm2.04$ & $30.29\pm4.06$ & $+21.46\pm4.99$ & $2.16\pm0.03$ \\
$0.40$ & $51.57\pm2.03$ & $28.36\pm4.23$ & $+23.21\pm5.36$ & $2.15\pm0.03$ \\
$0.50$ & $51.51\pm2.04$ & $27.48\pm4.32$ & $+24.03\pm5.60$ & $2.15\pm0.03$ \\
$0.60$ & $51.26\pm2.04$ & $26.75\pm4.40$ & $+24.51\pm5.71$ & $2.14\pm0.03$ \\
$0.75$ & $50.95\pm2.09$ & $25.96\pm4.49$ & $+24.99\pm5.99$ & $2.13\pm0.04$ \\
$0.90$ & $50.12\pm2.21$ & $25.35\pm4.54$ & $+24.77\pm6.35$ & $2.11\pm0.05$ \\
$0.95$ & $49.30\pm2.33$ & $25.18\pm4.56$ & $+24.12\pm6.51$ & $2.09\pm0.06$ \\
$1.00$ & $30.65\pm2.81$ & $24.97\pm4.54$ & $+5.68\pm2.00$ & $1.93\pm0.28$ \\
\bottomrule
\end{tabular*}
\end{table}

\begin{table}[H]
\centering\small
\caption{Prediction-rule comparison and inference costs using identical frozen models and paths.
Accuracy is in percent (three-seed mean $\pm$ sample SD).
Hard terminal uses the final visited expert; uniform and weighted history use only visited predictions.
Edge scores and classifier calls are mean logical operations per query, not measured latency.}
\label{tab:frozen200-inference}
\begin{tabular}{lrrrrr}
\toprule
Graph & Hard terminal & Uniform history & Weighted history & Calls & Incident scores \\
\midrule
\textsc{Path} & $11.09\pm0.60$ & $10.43\pm0.79$ & $\mathbf{11.19\pm0.58}$ & 1.72 & 2.35 \\
\textsc{Cycle} & $12.29\pm0.63$ & $11.62\pm0.79$ & $\mathbf{12.40\pm0.63}$ & 1.81 & 2.81 \\
\textsc{RR4} & $17.28\pm0.55$ & $15.62\pm0.47$ & $\mathbf{17.59\pm0.59}$ & 2.10 & 6.56 \\
\textsc{Complete} & $\mathbf{20.05\pm0.55}$ & $17.71\pm0.67$ & $20.02\pm0.63$ & 1.94 & 12.66 \\
\bottomrule
\end{tabular}
\end{table}

\FloatBarrier

The sweep in Table~\ref{tab:lambda-tradeoff} is a sensitivity analysis rather than test-set model selection.
Referral remains effective across the nonzero routing-loss settings, whereas removing the routing term at $\lambda=1$ weakens performance.

Table~\ref{tab:frozen200-inference} compares prediction rules on identical routes.
Weighted history achieves the highest mean accuracy on the sparse graphs and is comparable to hard-terminal prediction on \textsc{Complete}.

\subsection{ISIC2019: Data, Protocols, and Additional Results}
\label{app:isic2019}

The study in Section~\ref{sec:isic2019} uses a fixed 400-round shared-training prefix followed by 50 rounds of balanced-endpoint frozen-encoder adaptation, with three training seeds.

\subsubsection{Shared data, preprocessing, and graphs}
\paragraph{Data provenance and center identities.}
We use the 23,247 images and six source-center assignments in the ISIC2019 split distributed with FLamby~\citep{NEURIPS2022_232eee8e}.
Images come from the public ISIC2019 training collection, including BCN20000, HAM10000, and MSK contributions \citep{Hern_ndez_P_rez_2024,Tschandl_2018,codella2018skinlesionanalysismelanoma}.
In particular, the evaluation set here is FLamby's held-out subset of that collection, not the separate ISIC challenge test collection.
Center identities follow source and acquisition metadata and should not be interpreted as six independent hospitals.
The \texttt{FedSPM} package~\citep{wang2026fedspm} serves only as a reference for the data input pipeline; we do not use its within-server subpopulation model or its pooled resplitting procedure.
Each observed center is represented by one expert.

\paragraph{Leakage control and fixed split.}
The released FLamby partition contains 18,597 training and 4,650 test images.
Our metadata audit finds 2,302 lesion identifiers shared across these folds.
We preserve every test image and label and remove the 6,102 training images whose lesion identifier also occurs in the test fold.
We additionally link lesion identifiers through identical original or preprocessed image files and one visually confirmed near-duplicate pair.
These links are computed globally, including across centers; groups with conflicting labels are excluded from the training candidates.
The image-level audit removes four additional training candidates.
A deterministic grouped split (seed 42) targets a 20\% validation share within center--class cells, while reserving at least one training group for each represented cell.
The resulting counts are given in Table~\ref{tab:isic2019-split}.
Training, validation, and test are disjoint in the audited lesion/duplicate groups.
The grouping metadata do not establish patient-level independence, and the near-duplicate review is not an exhaustive guarantee against all transformed duplicates.
This modified split is not directly comparable to the unmodified FLamby leaderboard.

\par\medskip\noindent\begin{minipage}{\linewidth}
\centering
\captionsetup{hypcap=false}
\captionof{table}{Fixed ISIC2019 split after global lesion and identified duplicate grouping.
Excluded images come from the original training fold.
The $n_i$ used for the deployed sample-count offset are the final training counts.}
\label{tab:isic2019-split}
\small
\begin{tabular}{rlrrrr}
\toprule
Node & FLamby source & Train & Validation & Test & Excluded \\
\midrule
1 & BCN & 3,802 & 954 & 2,483 & 5,174 \\
2 & vidir\_molemax & 2,528 & 632 & 791 & 3 \\
3 & vidir\_modern & 1,686 & 424 & 672 & 581 \\
4 & rosendahl & 1,221 & 308 & 452 & 278 \\
5 & MSK4 & 520 & 131 & 164 & 4 \\
6 & vienna\_dias & 229 & 56 & 88 & 66 \\
\midrule
 & Total & 9,986 & 2,505 & 4,650 & 6,106 \\
\bottomrule
\end{tabular}

\end{minipage}\par\medskip

\paragraph{Images and model.}
We resize each RGB image with bilinear interpolation so that its shorter side is 224 pixels, apply the FLamby shades-of-gray color-constancy transform with power six, and save a JPEG with quality 75.
Training uses a random $200\times200$ crop, independent horizontal and vertical flips, rotation within $\pm50$ degrees, and brightness/contrast jitter of 0.15/0.1.
Validation and test use a deterministic center crop.
All roles use ImageNet channel normalization.
These torchvision transforms are shared by all methods; we do not claim exact reproduction of FLamby's augmentation code.
The model is the existing \texttt{FedDRM} GroupNorm ResNet-18 with its stride-one stem, a $512\to256$ projection with LayerNorm/ReLU, and a private eight-class linear head.
Within each seed, all six experts and all methods start from the same random initialization; no pretrained weights are used.
\texttt{EdgeDRM} adds a 256-dimensional slope and scalar intercept per node.

\paragraph{Communication graphs.}
The four undirected graphs---\textsc{Path}, \textsc{Cycle}, \textsc{RR3}, and \textsc{Complete}---are fixed before training.
\textsc{Path} has edges $\{i,i+1\}$ for $i=1,\ldots,5$, and \textsc{Cycle} adds the edge $\{1,6\}$.
The intermediate regular graph is
\begin{equation*}
E(\textsc{RR3})
=\{\{1,2\},\{1,5\},\{1,6\},\{2,3\},\{2,4\},\{3,5\},\{3,6\},\{4,5\},\{4,6\}\},
\end{equation*}
\textsc{Complete} contains all pairs.
Their edge counts are 5, 6, 9, and 15 and their diameters are 5, 3, 2, and 1.
We use symmetric Metropolis mixing, with $W_{ij}=1/(1+\max\{\deg(i),\deg(j)\})$ on an edge and the diagonal completing each row to one.
Graph diameter alone does not ensure that the positive-gain, no-revisit referral rule reaches a globally best expert.

\subsubsection{Training Details}
\label{app:isic2019-current-protocol}

\paragraph{Shared-training prefix.}
Each graph uses six identically initialized local models, three training seeds, and 400 rounds of task-only SGD followed by full-model Metropolis mixing.
There are ten local optimizer steps per round, effective batch size 32, and microbatch size 16.
The initial learning rate is 0.03 with a 400-round cosine horizon; SGD momentum is 0.9, reset each round, and weight decay is $5\times10^{-4}$.
The \textsc{RR3} prefix reuses the unchanged final checkpoint of the earlier \texttt{DFedAvgM} run, before its evaluation-only consensus.
The other graphs require their own task-only prefixes; old joint-training \texttt{EdgeDRM} checkpoints are not substituted for them.
The comparison uses fixed round 400, not the historical validation-selected checkpoint.

\paragraph{Frozen adaptation.}
The backbone, feature projection, feature normalization, and their buffers are fixed to the prefix checkpoint and evaluated in inference mode.
Private classifier heads inherit their prefix values, and all node routing potentials start at zero.
For 50 rounds, five local mini-batches per round update only the heads and balanced-endpoint routing potentials, using equal endpoint weights and $\lambda=0.9$.
Microbatch gradients are weighted by their fractions of the full batch.
Endpoint owners accumulate both local and neighbor contributions from a common snapshot before updating.
No encoder or head mixing is performed, and no extra consensus phase is added.
The head/router learning rate starts at 0.01 and uses \eqref{eq:two-stage-lr} with horizon 600, stopping at round 50; momentum, weight decay, and batch sizes are as above.
Training retains random image augmentation despite the frozen representation parameters.

\paragraph{Selection, inference, and budget interpretation.}
Validation readouts are saved at adaptation rounds 0, 10, 25, and 50.
Within each topology, test comparisons use the fixed prefix and round-50 endpoint without selecting a model from the test scores.
The main table reports \textsc{Path}, \textsc{Cycle}, \textsc{RR3}, and \textsc{Complete}; topology comparisons are descriptive rather than a preregistered test.
The full prefix uses 4,000 local optimizer steps per client; adaptation adds 250 head/router steps.
Thus comparison with \texttt{DFedAvgM400} estimates the change from adding adaptation and referral to that prefix, and does not establish superiority at equal total training effort.
The adapted entry, hard terminal, uniform history, and score-weighted history controls use identical final model parameters and matched queries.
The graph, split, loss weights, and inference protocol are fixed before each run.
Each graph has three fixed training seeds.

\paragraph{Metric and access checks.}
For a prediction rule $\widehat y(x,s)$ and the test images $\mathcal T_k$ of class $k\in[8]$, the uniform-entry metric is
\[
\operatorname{BA}(\widehat y)
=\frac18\sum_{k=1}^8\frac1{6|\mathcal T_k|}
 \sum_{(x,y)\in\mathcal T_k}\sum_{s=1}^6
 \mathbf1\{\widehat y(x,s)=y\}.
\]
For the primary rule, $\widehat y$ is the argmax of the visited-history probability mixture.
Each of 4,650 unique test images contributes six entry records, for 27,900 records per model.
These records are paired comparisons of repeated images, not independent test subjects.
The deployed router adds the single sample-count offset $\log(n_j/n_i)$ to the balanced-endpoint logit, and the cumulative history weights use the resulting scores without another correction.
Its zero threshold and local frozen representations do not establish the additional statistical or routing assumptions in Section~\ref{sec:theory}.

\end{document}